\documentclass[journal]{IEEEtran}

\usepackage{amssymb}
\usepackage{graphicx}

\usepackage{mathrsfs}
\usepackage{amsfonts}
\usepackage{amsmath}
\usepackage{color}
\usepackage{cite}
\usepackage{fancyhdr}
\usepackage[ruled]{algorithm}
\usepackage{algcompatible}

\mathchardef\mhyphen="2D
\newtheorem{thm}{Theorem}
\newtheorem{lemma}{Lemma}

\newtheorem{prop}{Proposition}

\begin{document}

\title{Recursive Sparse Point Process Regression with Application to Spectrotemporal Receptive Field Plasticity Analysis}

\author{Alireza~Sheikhattar,~Jonathan~B.~Fritz, Shihab~A.~Shamma,~and~Behtash~Babadi
\thanks{A. Sheikhattar is with the Department of Electrical and Computer Engineering (ECE), University of Maryland, College Park, MD; J. B. Fritz is with the Institute for Systems Research (ISR), University of Maryland, College Park, MD; S. A. Shamma and B. Babadi are with ECE and ISR, University of Maryland, College Park, MD (e-mails: arsha89@umd.edu; ripple@umd.edu; sas@umd.edu; behtash@umd.edu).}%
\thanks{Corresponding author: B. Babadi (e-mail: behtash@umd.edu).}}

\maketitle

\begin{abstract}
We consider the problem of estimating the sparse time-varying parameter vectors of a point process model in an online fashion, where the observations and inputs respectively consist of binary and continuous time series. We construct a novel objective function by incorporating a forgetting factor mechanism into the point process log-likelihood to enforce adaptivity and employ $\ell_1$-regularization to capture the sparsity. We provide a rigorous analysis of the maximizers of the objective function, which extends the guarantees of compressed sensing to our setting. We construct two recursive filters for online estimation of the parameter vectors based on proximal optimization techniques, as well as a novel filter for recursive computation of statistical confidence regions. Simulation studies reveal that our algorithms outperform several existing point process filters in terms of trackability, goodness-of-fit and mean square error. We finally apply our filtering algorithms to experimentally recorded spiking data from the ferret primary auditory cortex during attentive behavior in a click rate discrimination task. Our analysis provides new insights into the time-course of the spectrotemporal receptive field plasticity of the auditory neurons. 
\end{abstract}


\section{Introduction}

Analyses of spiking activity recorded from sensory neurons have revealed three main features: first, neuronal activity is stochastic and exhibits significant variability across trials; second, the spiking statistics often undergo rapid changes referred to as neuronal plasticity, in order to adapt to changing stimulus salience and behavioral context; and third, the tuning characteristics of sensory neurons to the stimuli exhibit a degree of sparsity. Examples include place cells in the hippocampus \cite{frank2004hippocampal} and spectrotemporally tuned cells in the primary auditory cortex \cite{depireux2001spectro}. Hence, in order to gain insight into the functional mechanism of the underlying neural system, it is crucial to have a mathematical theory to simultaneously capture the stochasticity, dynamicity and sparsity of neuronal activity.

On one hand, the theory of point processes \cite{Daley} has been recently adopted as a mathematical framework to model the stochasticity of neuronal data. Traditionally, these models have been used to predict the likelihood of self-exciting processes such as earthquake occurrences \cite{ogata1988statistical, vere1970stochastic}, but have recently found significant applications in the analysis of neuronal data \cite{brown2004multiple, brown2001analysis, smith2003estimating, paninski2004maximum, paninski2007statistical, pillow2011model, truccolo2005point}.

On the other hand, classic results in signal processing such as the Least Mean Squares (LMS) and Recursive Least Squares (RLS) algorithms \cite{haykin2008adaptive} have created a framework to efficiently capture the dynamics of the parameters in linear observation models. Existing solutions in computational neuroscience have adopted this framework to estimate the dynamics of neuronal activity. For instance, in \cite{brown2001analysis} an LMS-type point process filter was introduced to study plasticity in hippocampal neurons. In \cite{eden2004dynamic}, more general adaptive filtering solutions based on approximations to the Chapman-Kolmogorov equation were introduced. Although quite powerful in analyzing neuronal data, these solutions do not account for the sparsity of the underlying parameters.

Finally, the theory of compressed sensing (CS) has provided a novel methodology for measuring and estimating statistical models governed by sparse underlying parameters \cite{donoho2006compressed, candes2006compressive, candes2008introduction}. In particular, for \emph{static} linear and generalized linear models (GLM) with random covariates and sparsity of the parameters, the CS theory characterizes sharp trade-offs between the number of measurement, sparsity, and estimation accuracy \cite{candes2006compressive, Negahban}. The sparse solutions of CS are typically achieved using batch-mode convex programs and greedy techniques. In online settings, sparse adaptive filters have only been introduced in the context of linear systems governed by sparse parameters such as communication channels \cite{babadi2010sparls,kalouptsidis2011adaptive,dumitrescu2012greedy}. 

Despite significant progress in all these research fronts, a unified framework to simultaneously capture the stochasticity, dynamicity and sparsity of neuronal data is lacking. In this paper, we close this gap by integrating techniques from point process theory, adaptive filtering, and compressed sensing. To this end, we consider the problem of estimating time-varying stimulus modulation coefficients (e.g., receptive fields) from a sequence of binary observations in an online fashion. We model the spiking activity by a conditional Bernoulli point process, where the conditional intensity is a logistic function of the stimulus and its time lags. We will then design a novel objective function by incorporating the forgetting factor mechanism of RLS-type algorithms into the $\ell_1$-regularized maximum likelihood estimation of the point process parameters. We will present theoretical guarantees that extend those of CS theory and characterize fundamental trade-offs between the number of measurements, forgetting factor, model compressibility, and estimation error of the underlying point processes in the non-asymptotic regime. We will next develop two adaptive filters for recursive estimation of the objective function based on proximal gradient techniques, as well as a filter for recursive computation of statistical confidence regions.

In order to validate our algorithms, we provide simulation studies which reveal that the proposed adaptive filtering algorithms significantly outperform existing point process filters in terms of goodness-of-fit, mean square error and trackability. We finally apply our proposed filters to multi-unit spiking data from ferret primary auditory cortex (A1) during passive stimulus presentation and during performance of a click rate discrimination task \cite{fritz2003rapid} in order to characterize the spectrotemporal receptive field (STRF) plasticity of A1 neurons. Application of our algorithm to these data provides new insights into the time course of attention-driven STRF plasticity, with over 3 orders of magnitude increase in temporal resolution from minutes to centiseconds, while capturing the underlying sparsity in a robust fashion. Aside from their theoretical significance, our results are particularly important in light of the recent technological advances in neural prostheses, which require real-time robust neuronal system identification from limited data.

The rest of the paper is organized as follows: In Section \ref{prelim}, we present our notational conventions, preliminaries and problem formulation. In Section \ref{sec:main}, we introduce the main theoretical results of this paper, including the construction and stability analysis of the objective function, recursive filter development, and computation of confidence regions. Section \ref{sec:applications} provides numerical simulations as well as application to real data, followed by our concluding remarks in Section \ref{sec:conclusion}. Technical details of Section \ref{sec:main} are presented in Appendices \ref{proof}--\ref{app:conf}.

\section{Preliminaries and Problem Definition}\label{prelim}

We {first} give a brief introduction to point process models (see \cite{Daley} for a detailed treatment). We will use the following {notation} throughout the paper. Parameter vectors are denoted by bold-face greek letters. For example, $\boldsymbol{\omega}=[\omega_1,\omega_2,\cdots,\omega_M]'$ denotes an $M$-dimensional parameter vector, with $[\cdot]'$ denoting the transpose operator.

Consider a stochastic process defined by a sequence of discrete events at random points in time, noted by $t_1^J = [t_1,t_2,\cdots,t_J]'$, and a counting measure {given by}
\begin{equation*}
\label{sequence}
dN(t) = \sum_{k=1}^J\delta(t-t_k), \quad \text{and} \quad N(t) = \int_0^tdN(u),
\end{equation*}
where $\delta(.)$ is the Dirac's measure. The Conditional Intensity Function (CIF) for this process, denoted by $\lambda_{t|{H_t}}$, is defined as
\begin{equation}
\label{CIF_def}
\lambda({t|H_t})  := \lim_{\varepsilon \rightarrow0}\frac{\mathbb{P}\left(N(t+\varepsilon)-N(t)=1|H_t\right)}{\varepsilon},
\end{equation}
where  $H_t$ denotes the history of the process as well as the covariates up to time $t$. {The CIF can be interpreted as the \textit{instantaneous rate} given the history of the process and the covariates. A point process with a CIF  given by $\lambda({t|H_t})$ is defined as:}
\begin{enumerate}
\item $N(0)=0$
\item Given $0=t_0<t_1<t_2<\cdots$, the random variables $N(t_k)-N(t_{k-1})$ are \textit{conditionally} mutually independent.
\item For any $0\leq t_1<t_2$, $N(t_2)-N(t_1)$ is a Poisson random variable with probability distribution
\end{enumerate}
\begin{equation}
\nonumber \mathbb{P}\Big(N(t_2)-N(t_1)=k \Big) = \frac{\left(\int_{t_1}^{t_2} \lambda({t|H_t})dt\right)^k e^{-\int_{t_1}^{t_2} \lambda({t|H_t})dt}}{k!}.
\end{equation}
A point process model is fully characterized by its CIF. For instance, $\lambda(t | H_t) = \lambda$ corresponds to the homogenous Poission process with rate $\lambda$. A discretized version of this process can be obtained by binning $N(t)$ within an observation interval of $[0, \mathcal{T}]$ by bins of length $\Delta$, that is
\begin{equation}
\label{discrete}
n_t := N(t\Delta)-N((t-1)\Delta), \;\;\;\; t=1,2,\cdots,T
\end{equation}
where $T := \lceil \mathcal{T} / \Delta \rceil$. Throughout this paper, $\{n_t\}_{t=1}^T$ will be considered as the observed spiking sequence, which will be used for estimation purposes. Also, by approximating Eq. (\ref{CIF_def}) for small $\Delta \ll 1$, and defining $\lambda_t := \lambda(t \Delta | H_{t\Delta})$, we have:
\begin{equation}
\label{bernoulli}
   \begin{array}{ll}
       \mathbb{P}(n_t=0) =  1-\lambda_{t}\Delta + o(\Delta),\\
	   \mathbb{P}(n_t=1) =  \lambda_{t}\Delta + o(\Delta),\\       
       \mathbb{P}(n_t\geq2) = o(\Delta).
   \end{array}
\end{equation}
In discrete time, the orderliness of the process is equivalent to the requirement that with high probability not more than one event fall into any given bin. In practice, this can always be achieved by choosing $\Delta$ small enough. An immediate consequence of Eq. (\ref{bernoulli}) is that $\{ n_t\}_{t=1}^T$ can be approximated by a sequence of Bernoulli random variables with success probabilities $\{ \lambda_t \Delta \}_{t=1}^T$. 

A popular class of models for the CIF is given by Generalized Linear Models (GLM). In its general form, a GLM consists of two main components: an observation model (which is given by (\ref{bernoulli}) in this paper) and an equation expressing some (possibly nonlinear) function of the observation mean as a \textit{linear} combination of the covariates. In neuronal systems, the covariates consist of extrinsic covariates (e.g., neural stimuli) as well as intrinsic covariates (e.g., the history of the process). In this paper, we only consider GLM models with purely extrinsic covariates, although most of our results can be generalized to incorporate intrinsic covariates as well.

Let $s_t$ denote the stimulus at time bin $t$, $[\theta_0,\theta_1,\cdots,\theta_{M-2}]'$ denote the vector of stimulus modulation parameters, and $\mu$ denote the baseline firing rate. We adopt a logistic regression model for the CIF as follows:
\begin{align}
\operatorname{logit}(\lambda_t \Delta) := \log\left(\frac{\lambda_t \Delta}{1-\lambda_t \Delta}\right) = \mu + \sum_{i=0}^{M-2} \theta_i s_{t-i}
\label{logit1}
\end{align}
By defining $\boldsymbol{\omega} := [\mu, \theta_0,\theta_1,\cdots,\theta_{M-2}]'$ and $\mathbf{x}_t := [1, s_t,\cdots,s_{t-M+2}]$, we can equivalently write:
\begin{align}
\lambda_t \Delta = \operatorname{logit}^{-1}(\boldsymbol{\omega}' \mathbf{x}_t) := \frac{\exp(\boldsymbol{\omega}' \mathbf{x}_t)}{1+ \exp(\boldsymbol{\omega}' \mathbf{x}_t)}
\label{logit2}
\end{align}

The model above is also known as the logistic-link CIF model. Another popular model in the computational neuroscience literature is the log-link model where $\lambda_t \Delta = \exp(\boldsymbol{\omega}' \mathbf{x}_t)$. The significance of the logistic-link model is that $\operatorname{logit}^{-1}(.)$ maps the real line $(-\infty, +\infty)$ to the unit probability interval $(0,1)$, making it a feasible model for describing statistics of binary events independent of the scaling of the covariates and modulation parameters.

Despite capturing the stimulus dependence in quite a general form, the GLM model in (\ref{logit2}) represents a static model. We therefore generalize this model to the dynamic setting by allowing temporal variability of the modulation parameters:
\begin{align}
\lambda_t \Delta = \operatorname{logit}^{-1}(\boldsymbol{\omega}_t' \mathbf{x}_t) = \frac{\exp(\boldsymbol{\omega}_t' \mathbf{x}_t)}{1+ \exp(\boldsymbol{\omega}_t' \mathbf{x}_t)}
\label{GLMdynamic}
\end{align}
where $\boldsymbol{\omega}_t := [\mu_t, \theta_{t,0},\theta_{t,1},\dots,\theta_{t,M-2}]'$ represents the time-varying parameter vector at time $t$. Throughout the rest of the paper, we refer to $\mathbf{x}_t$ and $\boldsymbol{\omega}_t$ as the covariate vector and the modulation parameter vector at time $t$, respectively. 

In our applications of interest, the modulation parameter vector $\boldsymbol{\omega}$ exhibits a degree of sparsity \cite{Brown_pp, brown_func_conn}. That is, only certain components in the stimulus modulation have significant contribution in determining the statistics of the process. These components can be thought of as the preferred or intrinsic tuning features of the underlying neuron. To be more precise, for a sparsity level $L < M$, we denote by $S \subset \{1,2,\cdots,M \}$ the support of the $L$ highest elements of $\boldsymbol{\omega}$ in absolute value, and by $\boldsymbol{\omega}_L$ the best $L$-term approximation to $\boldsymbol{\omega}$. We also define
\begin{equation}
\sigma_L(\boldsymbol{\omega}) := \|\boldsymbol{\omega}-\boldsymbol{\omega}_L\|_1
\end{equation}  
to capture the compressibility of the parameter vector $\boldsymbol{\omega}$. Recall that for $\mathbf{x} \in \mathbb{R}^M$, the $\ell_1$-norm is defined as $\| \mathbf{x} \|_1 := \sum_{i=1}^M |x_i|$. When $\sigma_L(\boldsymbol{\omega}) = 0$, the parameter $\boldsymbol{\omega}$ is called $L$-sparse, and when $\sigma_L(\boldsymbol{\omega})$ is small compared to $\| \boldsymbol{\omega}_L \|_1$, the parameter is called $L$-compressible \cite{needell2009cosamp}.

Finally, the main estimation problem of this paper can be stated as follows: \emph{given binary observations $\{n_{t}\}_{t = 1}^T$ and covariates $\{ \mathbf{x}_t \}_{t={-M+1}}^T$ from a point process with a CIF given by Eq. (\ref{GLMdynamic}), the goal is to estimate the $M$-dimensional $L$-compressible parameter vectors $\{ \boldsymbol{\omega}_t \}_{t=1}^T$ in an online and stable fashion.}

\section{Main Results}\label{sec:main}
In this section, we will first describe the construction of an appropriate objective function for addressing our main estimation problem. We will then present a rigorous analysis of the maximizers of the objective function, which extends the results of CS to our setting. Next, we will introduce two adaptive filters to recursively maximize the objective function based on proximal gradient techniques. Finally, we will outline how statistical confidence bounds can also be constructed in a recursive fashion for our estimates. 
\vspace{-3mm}
\subsection{$\ell_1$-regularized Exponentially Weighted Maximum Likelihood (ML)}
Before proceeding with the construction of the objective function, we need to introduce more notational conventions. In order to have a framework allowing multi-timescale dynamics, we consider piece-wise constant dynamics for the parameter $\boldsymbol{\omega}_t$. That is, we assume that $\boldsymbol{\omega}_t$ remains constant over windows of arbitrary length $W \ge 1$ samples, for some integer $W$. By segmenting the corresponding spiking data $\{n_t\}_{t=1}^{T}$ into $K := \frac{T}{W}$ windows of length $W$ samples each, we assume that the CIF for each time point $(k-1)W + 1 \le t \le kW$ is governed by $\boldsymbol{\omega}_t = \boldsymbol{\omega}_k$, for $k=1,2,\cdots,K$. Note that number of spiking samples $K$ is assumed to be an integer multiple of window size $W$, without loss of generality.

Invoking the Bernoulli approximation to the Poisson statistics for $\Delta \ll 1$, the log-likelihood of the observation $n_t$ at time $t$ can be expressed as:
\begin{align}
\nonumber \log p ( n_t ) & \approx n_t \log (\lambda_t \Delta) + (1-n_t) \log ( 1- \lambda_t \Delta) \\
& = n_t (\mathbf{x}'_t \boldsymbol{\omega}_t) - \log \left ( 1 + \exp \left ( \mathbf{x}'_t \boldsymbol{\omega}_t \right) \right ).
\end{align}
Assuming conditional independence of the spiking events, the joint log-likelihood of the observations within window $k$ is given by: 
\begin{align}
\nonumber \mathcal{L}_k(\boldsymbol{\omega}_k) := \sum_{j=1}^W  & n_{(k-1)W + j} \mathbf{x}_{(k-1)W+j}'\boldsymbol{\omega}_k\\
& - \log \big(1 + \exp( \mathbf{x}_{(k-1)W+j}' \boldsymbol{\omega}_k)\big) \label{LLi}
\end{align}

In order to explicitly enforce adaptivity in the log-likelihood function, we adopt the forgetting factor mechanism of the RLS algorithm, where the log-likelihood of each window is exponentially weighted regressively in time, with a forgetting factor $0 < \beta \le 1$. That is, the effective data log-likelihood up to and including window $k$ is taken to be:
\begin{equation}
\mathcal{L}^{\beta}(\boldsymbol{\omega}_k) := \sum_{i=1}^k \beta^{k-i} \mathcal{L}_i(\boldsymbol{\omega}_k) \label{LL}
\end{equation}
for some $0 < \beta \le 1$. Note what for $\beta = 1$, $\mathcal{L}^1(\boldsymbol{\omega}_k)$ coincides with the natural data log-likelihood. Moreover, if we replace the Bernoulli log-likelihood with the Gaussian log-likelihood, then $\mathcal{L}^\beta(\boldsymbol{\omega}_k)$ coincides with the conventional RLS objective function.

Next, in order to explicitly enforce sparsity, we adopt the $\ell_1$-regularization mechanism of CS. That is, at window $k$, we seek an estimate of the form:

\begin{equation}
\widehat{\boldsymbol{\omega}}_k = \underset{\boldsymbol{\omega}_k}{\operatorname{argmax}} \quad \left \{ \mathcal{L}^{\beta}(\boldsymbol{\omega}_k) - \gamma \|  \boldsymbol{\omega}_k\|_1 \right \}
\label{MainProb}
\end{equation}
where $\gamma$ is a regularization parameter controlling the trade off between the log-likelihood fit and the sparsity of estimated parameters. Our theoretical analysis in the next subsection reveals appropriate choices for $\gamma$, $\beta$ and the trade-offs therein.

\vspace{-3mm}
\subsection{Stability Analysis of the Objective Function}\label{maintheorem}

In order to quantify the trade-offs involving our choice of the objective function in Eq. (\ref{MainProb}), we proceed in the tradition of performance analysis result of the RLS algorithm  \cite{haykin2008adaptive} by characterizing the geometric properties of the estimates $\boldsymbol{\omega}_k$ in a stationary environment where $\boldsymbol{\omega_k} = \boldsymbol{\omega}$ for all $k$. Our analysis, however, is quite general and avoids ad hoc assumptions such as direct averaging or covariate independence which are usually invoked in the analysis of least squares problems.

We need to make the following technical assumptions for our analysis:

1) The stimulus sequence $\{s_t\}_{t=-M+1}^T$ consists of independent (but not necessarily identically distributed) random variables with a variance of $\sigma^2$ which are uniformly bounded by a constant  $B > 0$ in absolute value. Note that with this assumption, two successive covariate vectors, say at times $t$ and $t+1$, given respectively by $\mathbf{x}_t = [1, s_{t-M+2}, s_{t-M+3}, s_{t-M+4},\cdots, s_t]$ and $\mathbf{x}_{t+1} = [1, s_{t-M+3}, s_{t-M+4}, \cdots, s_t, s_{t+1}]$ are highly \emph{dependent}, as they have $M-3$ random variables in common. Hence, the independent assumption used in studying least squares problem is violated. 

2) We further assume that for all times $t$, $0 < p_{\min} \le \lambda_t \Delta \le p_{\max} < 1$, for some constants $p_{\min}$ and $p_{\max}$, i.e., the probability of spiking does not reach its extremal values of $0$ and $1$, but can get arbitrarily close. This assumption can be realized due to the boundedness of the covariates and appropriate normalization of the stimulus modulation coefficients, and does not result in any practical loss of generality.

We have the following theoretical result regarding the stability of the maximizers of the objective function:
\medskip

\begin{thm}\label{thm:main}
\textit{
Suppose that binary observations from a point process with a CIF given by Eq. (\ref{GLMdynamic}) are given over $K$ windows of length $W$ each. Consider a stationary environment with $\boldsymbol{\omega}_k = \boldsymbol{\omega}$ for all $k$ and suppose that $\boldsymbol{\omega}$ is $L$-compressible. Then, under assumptions (1) and (2), for a fixed positive constant $d > 0$, there exist constants $C$, $C'$, and $C''$ such that for $1 - \frac{C'}{L^2 \log M} \le \beta <1$, $K \ge \frac{\log 2}{\log (\frac{1}{\beta})}$ and a choice of $\gamma \ge C'' \sqrt{\frac{ \log M}{1-\beta}}$, any solution $\widehat{\boldsymbol{\omega}}$ to (\ref{MainProb}) satisfies the bound
\begin{align}
\nonumber \left \|\widehat{\boldsymbol{\omega}}\!-\!\boldsymbol{\omega}\right \|_2  \! \leq \!C \sqrt{\!(1\!-\!\beta) L \log M} \!+\! \sqrt{C \sigma_L(\boldsymbol{\omega})}\sqrt[4]{\!(1\!-\!\beta) L \log M},
\end{align}
with probability at least $1-\frac{5}{M^d}$. The constants $C$, $C'$, and $C''$ are only functions of $d$, $p_{\min}$, $p_{\max}$, $\sigma^2$, $B$, and $W$, and are explicitly given in Appendix \ref{proof}.}
\end{thm}
\medskip

\begin{proof}
The proof is given in Appendix \ref{proof}.
\end{proof}
\medskip

\textbf{Remarks.} The result of Theorem \ref{thm:main} has four major implications. First, the error bound scales with $\sqrt{L}$, the sparsity level, as opposed to $M$, the ambient dimension of the parameter vector, which is consistent with results from CS, and results in the robustness of the estimate when the underlying parameter is sparse. Note that the bounds holds for general non-sparse  $\boldsymbol{\omega}$, but is sharpest when $\sigma_L(\boldsymbol{\omega})$ is negligible, i.e., the parameter vector is nearly $L$-sparse.

Second, the theorem prescribes a lower bound on the forgetting factor akin to the bounds obtained in CS theory for the total number of observations. For instance, the result of \cite{Toeplitz} for CS under Toeplitz sensing measurements for the linear model requires $T = \mathcal{O}(L^2 \log M)$ number of measurements to achieve a similar scaling of the error bound. In our case, the role of the number of measurements is transferred to forgetting factor by taking $\frac{1}{1-\beta}$ as the \emph{effective} length of the measurements. In the absence of the forgetting factor ($\beta = 1$), by a careful limiting process, our results require $T = \mathcal{O}(L^2 \log M)$ measurements. The latter case can be compared to the result of \cite{Negahban} for point process models with independent and identically distributed covariate vectors, which requires $\mathcal{O}(L \log M)$ for stability. The loss of $\mathcal{O}(L)$ is incurred due to the shift structure and hence high dependence of the covariate vectors in our case, as exemplified in assumption (1).

Third, the theorem reveals the scaling of the regularization parameter in terms of $M$ and $\beta$. In particular, this scaling is significant as it reveals another role for the forgetting factor mechanism: not only the forgetting factor mechanism allows for adaptivity of the estimates, it also controls the scaling of the $\ell_1$-regularization term with respect to the log-likelihood term.  Fourth, unlike conventional results in the analysis of adaptive filters which concern the expectation of the error in the asymptotic regime, our result holds for a single realization with probability polynomially approaching 1, in the non-asymptotic regime.

Note that the objective function is clearly concave, and assuming that the matrix of the covariate vectors is full-rank, will be strictly concave with a unique maximizer. However, the result of Theorem \ref{thm:main} does not require the uniqueness of the maximizer and holds for any maximizer of the objective function. In the next section, we will proceed with the development of recursive filters to track the maximizer of the objective function in the more general time-varying setting.

\vspace{-2mm}
\subsection{Algorithm Development}
Several standard optimization techniques, such as interior point methods, can be used to find the maximizer of (\ref{MainProb}). However, most of these techniques operate in batch mode and do not meet the real-time requirements of the adaptive filtering setting where the observations arrive in a streaming fashion. In order to avoid the increasing runtime complexity and memory requirements of the batch-mode computation, we seek a recursive approach which can perform low-complexity updates in an online fashion upon the arrival of new data in order to form the estimates. To this end, we adopt the proximal gradient approach. A version of the proximal gradient algorithm is given in Appendix \ref{proximal}. Each iteration of the algorithm moves the previous iterate along the gradient of the log-likelihood function, which will then pass through a shrinkage operator. 

Before describing further details, we introduce a more compact notation for convenience. Let $\mathbf{n}_k := [n_{(k-1)W+1},n_{(k-1)W+1},\dots,n_{kW} ]'$ denote the vector of observed spikes within window $k$, for $k = 1,2,\dots,K$. Similarly, let $\boldsymbol{\lambda}_k := \big[\lambda_{(k-1)W+1},\lambda_{(k-1)W+1},\dots,\lambda_{kW} \big]'$ denote the vector of CIFs within window $k$. By extending the domain of the $\operatorname{logit}^{-1} (\cdot)$ to vectors in a component-wise fashion, we can express $\boldsymbol{\lambda}_k \Delta$ as:
\begin{align}
 \boldsymbol{\lambda}_k \Delta = \operatorname{logit}^{-1} \Big( \mathbf{X}_k \boldsymbol{\omega}_k \Big)
\end{align}
\vspace{-2mm}

\noindent where $\mathbf{X}_k := \big[\mathbf{x}_{(k-1)W+1},\mathbf{x}_{(k-1)W+2},\dots,\mathbf{x}_{kW} \big]'$ is the data matrix of size $W \times M$ with rows corresponding to the covariate vectors in window $k$.
Suppose that at window $k$, we have an iterate denoted by $\widehat{\boldsymbol{\omega}}_k^{(\ell)}$, for $\ell=0,1,\cdots, R$, with $R$ being an integer denoting the total number of iterations. The gradient of $\mathcal{L}^{\beta}(\cdot)$ evaluated at $\widehat{\boldsymbol{\omega}}_k^{(\ell)}$ can be written as:
\vspace{-1mm}
\begin{align}
 \nabla_{\boldsymbol{\omega}} {\mathcal{L}}^{\beta}\left(\widehat{\boldsymbol{\omega}}_k^{(\ell)}\right) &= \sum_{i = 1}^k \beta^{k-i} \mathbf{X}_i' \boldsymbol{\varepsilon}_i\left(\widehat{\boldsymbol{\omega}}_k^{(\ell)}\right) =: \mathbf{g}_k \left (\widehat{\boldsymbol{\omega}}_k^{(\ell)}\right ) \label{gradSerie} 
\end{align}
\vspace{-2mm}

\noindent where $\boldsymbol{\varepsilon}_i(\cdot) := \mathbf{n}_i - \boldsymbol{\lambda}_i(\cdot)\Delta$ represents the innovation vector of the point process at window $i$. The innovation vector $\boldsymbol{\varepsilon}_i$ can be thought of as the counterpart of the conventional innovation vector in adaptive filtering of linear models. The proximal gradient iteration for the $\ell_1$-regularization can be written in the compact form as:
\begin{align}
\widehat{\boldsymbol{\omega}}_k^{(\ell+1)} &= \mathcal{S}_{\gamma \alpha} \Big(\widehat{\boldsymbol{\omega}}_k^{(\ell)} + \alpha \mathbf{g}_k\left(\widehat{\boldsymbol{\omega}}_k^{(\ell)}\right) \Big)
\label{GenItr}
\end{align} 
where $\mathcal{S}_{\tau}(\cdot)$ is the element-wise soft thresholding operator at a level of $\tau$ given in Appendix \ref{proximal}. The final estimate at window $k$ is obtained following the $R$th iteration, and is denoted by $\widehat{\boldsymbol{\omega}}_k := \widehat{\boldsymbol{\omega}}_k^{(R)}$. In order to achieve a recursive updating rule for $\mathbf{g}_k$, we can rewrite Eq. (\ref{gradSerie}) as:
\begin{align}\label{eq:rec}
\mathbf{g}_k\left(\widehat{\boldsymbol{\omega}}_k^{(\ell)}\right) &= \beta \, \mathbf{g}_{k-1}\left(\widehat{\boldsymbol{\omega}}_k^{(\ell)}\right) + \mathbf{X}_k' \boldsymbol{\varepsilon}_k\left(\widehat{\boldsymbol{\omega}}_k^{(\ell)} \right)
\end{align} 
However, in an adaptive setting, we only have access to values of $\mathbf{g}_{k-1}$ evaluated at $\boldsymbol{\widehat{\omega}}^{(1:L)}_{k-1}(\cdot)$! In order to turn Eq. (\ref{eq:rec}) into a fully recursive updating rule, all the previous CIF vectors $\{ \boldsymbol{\lambda}_i(\cdot)\}_{i=1}^{k-1}$ should be recalculated at the most recent set of iterates $\boldsymbol{\widehat{\omega}}^{(1:L)}_{k}(\cdot)$. In order to overcome this computational burden, we exploit the smoothness of the logistic function and employ the Taylor series expansion of the CIF to approximate the required recursive update. In what follows, we consider the zeroth order and first order expansions, which result in two distinct, yet fully recursive, updating rules for Eq. (\ref{eq:rec}).
\medskip

\noindent \textit{\textbf{Zeroth Order Expansion:}} By retaining only the first term in the Taylor series expansion of the CIF $\boldsymbol{\lambda}_i\left(\widehat{\boldsymbol{\omega}}_k^{(\ell)}\right)$ around $\widehat{\boldsymbol{\omega}}_i$, we get:
\vspace{-2mm}
\begin{equation}\label{eq:0}
\boldsymbol{\lambda}_i\left(\widehat{\boldsymbol{\omega}}_k^{(\ell)}\right) \Delta \approx \boldsymbol{\lambda}_i\left(\widehat{\boldsymbol{\omega}}_i\right) \Delta
\end{equation})

\noindent where $\boldsymbol{\lambda_i}\Delta = \operatorname{logit}^{-1}(\mathbf{X}_i \widehat{\boldsymbol{\omega}}_i) $.  
Substituting this approximation in Eq. (\ref{gradSerie}), we can express the zeroth order approximation to the gradient at window $k$, denoted by $\mathbf{g}^{0}_k(\cdot)$, as:
\vspace{-2mm}
\begin{align}
\mathbf{g}^{0}_k\left(\widehat{\boldsymbol{\omega}}_k^{(\ell)}\right)  = \sum_{i = 1}^k \beta^{k-i} \, \mathbf{X}_i' \boldsymbol{\varepsilon}_i (\widehat{\boldsymbol{\omega}}_i)
\end{align}

\noindent It is then straightforward to obtain a recursive form as:
\begin{align}
\mathbf{g}^{0}_k \left(\widehat{\boldsymbol{\omega}}_k^{(\ell)}\right)&=  \beta \ \mathbf{g}^{0}_{k-1}\left(\widehat{\boldsymbol{\omega}}_k^{(\ell)}\right) + \mathbf{X}_k' \boldsymbol{\varepsilon}_k\left(\widehat{\boldsymbol{\omega}}_k^{(\ell)}\right)
\end{align}

\begin{figure}[b!]
\vspace{-6mm}
\noindent \begin{minipage}{\columnwidth}
\begin{algorithm}[H]
\caption{  {\small $\ell_1$-regularized Point Process Filter of the Zeroth Order} ($\ell_1$-PPF$_0$) }
\label{alg0}
\begin{algorithmic}[1]
\REQUIRE {$\mathbf{n}_k$, $\mathbf{X}_k$, $\mathbf{g}_{k-1}$, $\widehat{\boldsymbol{\omega}}_k^{(0)}$, and $R$ }.
\FOR{$\ell = 0, \dots, R-1$}
\STATE $\boldsymbol{\lambda}_k \Delta =  \operatorname{logit}^{-1} \left( \mathbf{X}_k \widehat{\boldsymbol{\omega}}_k^{(\ell)} \right)$
\STATE $\boldsymbol{\varepsilon}_k = \mathbf{n}_k - \boldsymbol{\lambda}_k \Delta$
\STATE $\mathbf{g}_k = \beta \, \mathbf{g}_{k-1} + \mathbf{X}_k' \boldsymbol{\varepsilon}_k$
\STATE $\widehat{\boldsymbol{\omega}}_k^{(\ell+1)} = \mathcal{S}_{\gamma \alpha} \Big[ \widehat{\boldsymbol{\omega}}_k^{(\ell)} +\alpha \mathbf{g}_k \Big] $
\ENDFOR
\ENSURE $\widehat{\boldsymbol{\omega}}_k := \widehat{\boldsymbol{\omega}}_k^{(R)}$.
\end{algorithmic}
\end{algorithm}
\end{minipage}
\end{figure}
\medskip

\noindent The shrinkage step will be then given by:
\begin{align}
\widehat{\boldsymbol{\omega}}_k^{(\ell+1)} &= \mathcal{S}_{\gamma \alpha} \Big(\widehat{\boldsymbol{\omega}}_k^{(\ell)} + \alpha \mathbf{g}^0_k\left(\widehat{\boldsymbol{\omega}}_k^{(\ell)}\right) \Big)
\end{align}

\noindent We refer to the resulting filter as the $\ell_1$-regularized Point Process Filter of the Zeroth Order ($\ell_1$-PPF$_0$). A pseudo-code is given in Algorithm \ref{alg0}.

\noindent \textit{\textbf{First Order Expansion:}} If instead, we retain the first two terms in the Taylor expansion, Eq. (\ref{eq:0}) will be replaced by:
\begin{equation}
\boldsymbol{\lambda}_i\left(\widehat{\boldsymbol{\omega}}_k^{(\ell)}\right) \Delta \approx \boldsymbol{\lambda}_i\left(\widehat{\boldsymbol{\omega}}_i\right) \Delta + \boldsymbol{\Lambda}_i \left ( \widehat{\boldsymbol{\omega}}_i \right ) \mathbf{X}_i \left ( \widehat{\boldsymbol{\omega}}^{(\ell)}_k - \widehat{\boldsymbol{\omega}}_i \right ) 
\end{equation}
where $\boldsymbol{\Lambda}_i(\widehat{\boldsymbol{\omega}}_i)$ is a diagonal $W \times W$ matrix with the $(m,m)$th diagonal element given by $\lambda_{(i-1)W+m} \Delta ( 1 - \lambda_{(i-1)W+m} \Delta)$. Using the first order approximation above, we can improve the resulting approximation to the gradient, denoted by $\mathbf{g}^{1}_k$, as:
\vspace{-2mm}
\begin{align}
\nonumber \mathbf{g}^{1}_k  \left(\widehat{\boldsymbol{\omega}}_k^{(\ell)}\right)= \sum_{i = 1}^k \beta^{k-i} \, \mathbf{X}_i' \Big(\boldsymbol{\varepsilon}_i(\widehat{\boldsymbol{\omega}}_i) -  \boldsymbol{\Lambda}_i(\widehat{\boldsymbol{\omega}}_i) \mathbf{X}_i \big(\widehat{\boldsymbol{\omega}}^{(\ell)}_k - \widehat{\boldsymbol{\omega}}_i\big) \Big)
\end{align}
By defining:
\vspace{-3mm}
\begin{align}\label{eq:hess}
\mathbf{u}_k^{(\ell)}&:= \sum_{i = 1}^k \beta^{k-i} \, \mathbf{X}_i' \Big(\boldsymbol{\varepsilon}_i (\widehat{\boldsymbol{\omega}}_i)+ \boldsymbol{\Lambda}_i(\widehat{\boldsymbol{\omega}}_i) \mathbf{X}_i \widehat{\boldsymbol{\omega}}_i \Big) \\
\mathbf{B}_k^{(\ell)} &:= \sum_{i = 1}^k \beta^{k-i} \, \mathbf{X}_i' \boldsymbol{\Lambda}_i(\widehat{\boldsymbol{\omega}}_i) \mathbf{X}_i,
\end{align}
\vspace{-3mm}

\noindent we can express $\mathbf{g}^{1}_k  \left(\widehat{\boldsymbol{\omega}}_k^{(\ell)}\right)$ as:
\begin{align}
\mathbf{g}^{1}_k  \left(\widehat{\boldsymbol{\omega}}_k^{(\ell)}\right)= \mathbf{u}^{(\ell)}_k - \mathbf{B}^{(\ell)}_k\widehat{\boldsymbol{\omega}}_k^{(\ell)}.
\end{align}

\begin{figure}[b!]
\vspace{-4mm}
\noindent \begin{minipage}{\columnwidth}
\begin{algorithm}[H]
\caption{  {\small $\ell_1$-regularized Point Process Filter of the First Order} ($\ell_1$-PPF$_1$) }
\label{alg1}
\begin{algorithmic}[1]
\REQUIRE {$\mathbf{n}_k$, $\mathbf{X}_k$, $\mathbf{u}_{k-1}$, $\mathbf{B}_{k-1}$, $\widehat{\boldsymbol{\omega}}_k^{(0)}$, and $R$ }.
\FOR{$\ell = 0, \dots, R-1$}
\STATE $\boldsymbol{\lambda}_k \Delta =  \operatorname{logit}^{-1} \left( \mathbf{X}_k \widehat{\boldsymbol{\omega}}_k^{(\ell)} \right)$
\STATE $\boldsymbol{\varepsilon}_k = \mathbf{n}_k - \boldsymbol{\lambda}_k \Delta$
\STATE $(\boldsymbol{\Lambda}_k)_{m,m} = (\boldsymbol{\lambda}_k)_m \Delta \left( 1- (\boldsymbol{\lambda}_k)_m \Delta \right)$, $m=1,\cdots,W$
\STATE $\mathbf{u}_k = \beta \, \mathbf{u}_{k-1} + \mathbf{X}_k' \left(\boldsymbol{\varepsilon}_k + \boldsymbol{\Lambda}_k \mathbf{X}_k \widehat{\boldsymbol{\omega}}_k^{(\ell)}  \right)$
\STATE $\mathbf{B}_k = \beta \, \mathbf{B}_{k-1} + \mathbf{X}_k' \boldsymbol{\Lambda}_k \mathbf{X}_k $
\STATE $\mathbf{g}_k = \mathbf{u}_k - \mathbf{B}_k \widehat{\boldsymbol{\omega}}_k^{(\ell)} $
\STATE $\widehat{\boldsymbol{\omega}}_k^{(\ell+1)} = \mathcal{S}_{\gamma \alpha} \Big[ \widehat{\boldsymbol{\omega}}_k^{(\ell)} +\alpha \mathbf{g}_k \Big] $
\ENDFOR
\ENSURE $\widehat{\boldsymbol{\omega}}_k := \widehat{\boldsymbol{\omega}}_k^{(R)}$.
\end{algorithmic}
\end{algorithm}
\end{minipage}
\end{figure}

It is then straightforward to check that both $\mathbf{u}_k$ and $\mathbf{B}_k$ can be updated recursively \cite{babadi2010sparls} as:
\begin{align}
\nonumber \mathbf{u}^{(\ell)}_k &= \beta \, \mathbf{u}^{(R)}_{k-1} + \mathbf{X}_k' \Big(\boldsymbol{\varepsilon}_k\left(\widehat{\boldsymbol{\omega}}_k^{(\ell)}\right) + \boldsymbol{\Lambda}_k\left(\widehat{\boldsymbol{\omega}}_k^{(\ell)}\right) \mathbf{X}_k \widehat{\boldsymbol{\omega}}_k^{(\ell)} \Big)  \\
\nonumber \mathbf{B}^{(\ell)}_k &= \beta \, \mathbf{B}^{(R)}_{k-1} + \mathbf{X}_k' \boldsymbol{\Lambda}_k\left(\widehat{\boldsymbol{\omega}}_k^{(\ell)}\right) \mathbf{X}_k 
\end{align}
Note that the update rules for both $\mathbf{B}^{(\ell)}_k$ and $\mathbf{u}^{(\ell)}_k$ involve simple rank-$W$ operations. The shrinkage step is then given by:
\vspace{-4mm}
\begin{align}
\widehat{\boldsymbol{\omega}}_k^{(\ell+1)} &= \mathcal{S}_{\gamma \alpha} \Big(\left( \mathbf{I} - \alpha \mathbf{B}_k^{(\ell)}\right ) \widehat{\boldsymbol{\omega}}_k^{(\ell)} + \alpha \mathbf{u}_k^{(\ell)} \Big)
\end{align}
We refer to the resulting filter as the $\ell_1$-regularized Point Process Filter of the First Order ($\ell_1$-PPF$_1$). A pseudo-code is given in Algorithm \ref{alg1}.
\medskip

\noindent \textbf{Remark}. The computational complexity of $\ell_1$-PPF$_0$ and $\ell_1$-PPF$_1$ algorithms can be shown to be linear and quadratic in $M$ per iteration, respectively. Our results in Section \ref{sec:applications} will reveal that both filters outperform existing filters of the same complexity, respectively. Furthermore, $\ell_1$-PPF$_1$ exhibits superior performance over $\ell_1$-PPF$_0$ as expected, although with a cost of $\mathcal{O}(M)$ in computational complexity per iteration.

\subsection{Constructing Confidence Intervals}

Characterizing the statistical confidence bounds for the estimates is of utmost importance in neural data analysis, as it allows to test the validity of various hypotheses. Although construction of confidence bounds for linear models in the absence of regularization is well understood and widely applied, regularized ML estimates are usually deemed as point estimates for which the construction of statistical confidence regions is not straightforward. A series of remarkable results in high-dimensional statistics \cite{javanmard2014confidence, van2014asymptotically, zhang2014confidence} have recently addressed this issue by providing techniques to construct confidence intervals for $\ell_1$-regularized ML estimates of linear models as well as GLMs. These approaches are based on a careful inspection of the Karush-Kuhn-Tucker (KKT) conditions for the regularized estimates, which admits a procedure to decompose the estimates into a bias term plus an asymptotically Gaussian term (referred to as 'de-sparsifying' in \cite{van2014asymptotically}), which can be computed using a nodewise regression \cite{meinshausen2006high} of the covariates.
\begin{figure}[b!]
\vspace{-4mm}
\noindent \begin{minipage}{\columnwidth}
\begin{algorithm}[H]
 \caption[]{Recursive Construction of the Confidence Regions for the $m$th Component of $\widehat{\mathbf{w}}_k$.\footnote{For a matrix $\mathbf{A} \in \mathbb{R}^{M \times M}$, we denote by $(\mathbf{A})_{m,\!\backslash m}$ the $m$th row with the $m$th element removed, and by $(\mathbf{A})_{\backslash m,\!\backslash m}$ the submatrix of $\mathbf{A}$ with both the $m$th row and column removed.}}
\label{alg3}
\begin{algorithmic}[1]
\REQUIRE {$\mathbf{n}_k$, $\mathbf{X}_k$, $\mathbf{u}_k$, $\mathbf{B}_k$, and $\widehat{\boldsymbol{\omega}}_k$, $\mathbf{G}_{k-1}$, $m$, $\gamma_m$, and $\widehat{\boldsymbol{\psi}}_m^{(0)}$. }
\STATE $\mathbf{g}_k = \mathbf{u}_k - \mathbf{B}_k \widehat{\boldsymbol{\omega}}_k  $
\STATE $\mathbf{G}_k = \beta^2 \mathbf{G}_{k-1} + \mathbf{X}_k' \boldsymbol{\varepsilon}_k \boldsymbol{\varepsilon}_k' \mathbf{X}_k $
\FOR{$\ell = 0, \dots, R-1$}
\STATE {$\widehat{\boldsymbol{\psi}}_{m}^{(\ell+1)} \!=\! \mathcal{S}_{\gamma_m \alpha} \Big[ \widehat{\boldsymbol{\psi}}_{m}^{(\ell)} \!- \alpha \Big((\mathbf{B}_{k})_{m,\!\backslash m} \!- (\mathbf{B}_{k})_{\backslash m, \!\backslash m} \widehat{\boldsymbol{\psi}}_{m}^{(\ell)} \Big)  \Big] $}
\ENDFOR
\STATE ${\tau}^2_{m} = (\mathbf{B}_k)_{m,m} - \widehat{\boldsymbol{\psi}}_{m}^{(R)} (\mathbf{B}_{k})'_{m,\!\backslash m} $
\STATE $(\mathbf{c})_m =1$, and $(\mathbf{c})_{\backslash m} = - \widehat{\boldsymbol{\psi}}_{m}^{(R)}$
\STATE $ (\widehat{\boldsymbol{\Theta}}_{k})_m = \frac{1}{{\tau}^2_{m}} \mathbf{c}$
\STATE $\widehat{\sigma}_{k,m}^2 := (\widehat{\boldsymbol{\Theta}}_{k})_m \mathbf{G}_k (\widehat{\boldsymbol{\Theta}}_{k})_m'$
\STATE $(\widehat{\mathbf{w}}_{k})_m = (\widehat{\boldsymbol{\omega}}_{k})_m - (\widehat{\boldsymbol{\Theta}}_{k})_m \mathbf{g}_k$
\ENSURE $ \mathcal{CR}_{k,m} := [ (\widehat{\mathbf{w}}_{k})_m \pm \Phi^{-1}(1-\alpha/2) \, \widehat{\sigma}_{k,m}]$
\end{algorithmic}
\end{algorithm}
\end{minipage}
\end{figure}

In what follows, we give a brief description of how the methods of \cite{van2014asymptotically} apply to our setting, and leave the details to Appendix \ref{app:conf}. Using the result of \cite{van2014asymptotically}, the estimate $\boldsymbol{\omega}_k$ as the maximizer of (\ref{MainProb}) can be decomposed as:
\begin{equation}\label{eq:decomp}
\widehat{\boldsymbol{\omega}}_k =  \widehat{\boldsymbol{\Theta}}_k \mathbf{g}_k (\widehat{\boldsymbol{\omega}}_k) + \widehat{\mathbf{w}}_k
\end{equation}
where $\widehat{\boldsymbol{\Theta}}_k$ is an approximate inverse to the Hessian of $\mathcal{L}^{\beta}(\boldsymbol{\omega})$ evaluated at $\widehat{\boldsymbol{\omega}}_k$, $\mathbf{g}_k$ is the gradient of $\mathcal{L}^{\beta}(\boldsymbol{\omega})$ previously defined in Eq. (\ref{gradSerie}), and $\widehat{\mathbf{w}}_k$ is an unbiased and asymptotically Gaussian random vector with a covariance matrix of ${\sf cov}(\widehat{\mathbf{w}}_k) = \widehat{\boldsymbol{\Theta}}_k \mathbf{G}_k(\widehat{\boldsymbol{\omega}}_k) \widehat{\boldsymbol{\Theta}}'_k$, with 
\begin{equation}
\mathbf{G}_k(\widehat{\boldsymbol{\omega}}_k) := \sum_{i=1}^k \beta^{2(k-i)} \mathbf{X}_i' \boldsymbol{\varepsilon}_i(\widehat{\boldsymbol{\omega}}_k) \boldsymbol{\varepsilon}_i(\widehat{\boldsymbol{\omega}}_k)' \mathbf{X}_i.
\end{equation}
The first term in Eq. (\ref{eq:decomp}) is a bias term which can be directly computed given $\widehat{\boldsymbol{\Theta}}_k$. Given ${\sf cov}(\widehat{\mathbf{w}}_k)$, statistical confidence bounds for the second term at desired levels can be constructed in a standard way. The main technical issue in the aforementioned procedure in our setting is the computation of $\widehat{\boldsymbol{\Theta}}_k$ in a recursive fashion. Since the rows of $\widehat{\boldsymbol{\Theta}}_k$ are computed using $\ell_1$-regularized least squares, we use the SPARLS algorithm \cite{babadi2010sparls} as an efficient method to carry out the computation in a recursive fashion. Algorithm \ref{alg3} summarized the recursive computation of confidence intervals for the $m$th component of $\widehat{\mathbf{w}}_k$.

\vspace{-3mm}
\section{Applications}\label{sec:applications}
In this section, we will apply the proposed algorithms to simulated data as well as real spiking data from the ferret primary auditory cortex. In our simulation studies, we compare the performance of our proposed filters with two of the state-of-the-art point process filters, namely the steepest descent point process filter (SDPPF) \cite{brown2001analysis} and the stochastic state point process filter (SSPPF) \cite{eden2004dynamic}. These adaptive filters are based on approximate solutions to the  Chapman-Kolmogorov forward equation obtained by a steepest descent and a Gaussian approximation procedure, respectively.

\vspace{-3mm}
\subsection{Simulation Study 1: MSE and Sparse Recovery Learning Curves}

First, we consider a stationary environment where $\boldsymbol{\omega}$ is constant over time. We use a bin size of $\Delta = 1~ms$ and window size of $W = 1$ sample, for a total observation window of $\mathcal{T} = 30~sec$ ($K= 30000$). The length of the parameter vector $\boldsymbol{\omega} = [\mu, \boldsymbol{\theta}]$ is chosen as $M=101$. For each realization, we draw a sparse parameter vector $\boldsymbol{\theta}$ of fixed length $M -1 = 100$ and sparsity $L = 3$. The support $S$ and values of the nonzero components of $\boldsymbol{\theta}$ are chosen randomly and the values are normalized so that $\| \boldsymbol{\theta} \|_2 = 10$. The stimulus input sequence $\{s_{k}\}_{k=-M+1}^K$ is drawn from an i.i.d. Gaussian distribution $\mathcal{N}(0,\sigma^2)$. The binary spike train $\{n_k \}_{k=1}^K$ is generated as a single realization of conditionally independent Bernoulli trials with success rate $\lambda_k \Delta$. The stimulus variance $\sigma^2$ is chosen as $\sigma^2 = 0.01$ small enough so that the average spiking rate $\lambda_{avg} \Delta = 0.13 \ll 1$ to ensure that the Bernoulli approximation is valid.  All the simulations are done with $R = 1$ iteration per time step. The step size $\alpha$ is chosen as $\alpha \simeq 9\times 10^{-4}$ (See Appendix \ref{proximal} for details). 

For a given pair of $(\alpha,\beta)$ parameters, we select an optimal value for the regularization parameter $\gamma$ by performing a two-fold even-odd cross validation procedure: first, the data are split into two sets of even and odd samples in an interleaved manner. Then, one set is used as the training set for estimation of the parameter vector $\boldsymbol{\omega}_k$ and the other is used to assess the goodness-of-fit of the estimates $\widehat{\boldsymbol{\omega}}_k$ with respect to the log-likelihood of the observations. We repeat the process switching the role of the two sets and take the average as the overall measure of fit. 

Let $\widehat{\mathbb{E}}$ denote the averaging operator with respect to realizations. We consider two performance metrics: the normalized mean squared error (MSE) defined as $\operatorname{MSE}_k := \widehat{\mathbb{E}} \|\widehat{\boldsymbol{\omega}}_k - \boldsymbol{\omega}_k \|^2/ \widehat{\mathbb{E}} \| \boldsymbol{\omega}_k \|^2 $ to evaluate MSE performance at time step $k$; and the out-of-support energy defined as $\operatorname{SPM}_k := \widehat{\mathbb{E}} \| \widehat{\boldsymbol{\theta}}_k - (\widehat{\boldsymbol{\theta}}_k)_S \|^2 / \mathbb{E} \| \widehat{\boldsymbol{\theta}}_k \|^2$ to represent a sparsity metric (SPM). Ideally, $\operatorname{SPM}_k$ must be equal to zero at all times. The averaging is carried out over a sufficiently large number of runs.

Figure \ref{fig:f1} shows the corresponding learning curves for the four algorithms.  According to Figure \ref{fig:f1}--A,  the $\ell_1$-PPF$_1$ achieves the lowest stationary MSE measure of $-10.4$ dB, followed $\ell_1$-PPF$_0$ which achieves an MSE of $-9.2$~dB. The SSPPF and SDPPF algorithms respectively achieve an MSE of $-3.35$ dB and $-1.9$ dB, which reveals a gap of $\approx 7$~dB with respect to our proposed filters. Note that this result is consistent with the prediction of Theorem \ref{thm:main} and highlights the MSE gain achieved by $\ell_1$-regularization, as opposed to ML, when the underlying parameter is sparse.

\begin{figure}
\centering
\includegraphics[width=0.85\columnwidth]{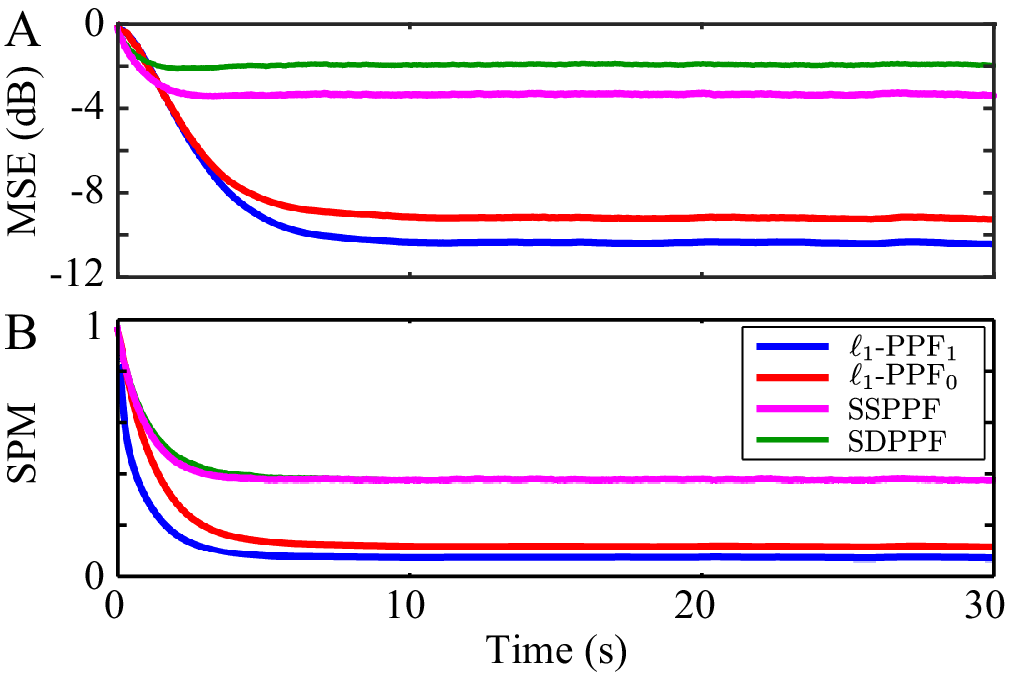}
\vspace{-2mm}
\caption{Learning curves of the adaptive filtering algorithms in a stationary environment. A) MSE vs. time, B) SPM vs. time.}
\label{fig:f1}
\vspace{-4mm}
\end{figure}

\vspace{-3.5mm}
\subsection{Simulation Study 2: Tracking and Goodness-of-fit Performance}

\begin{figure*}
\centering
\includegraphics[width= 0.8 \textwidth]{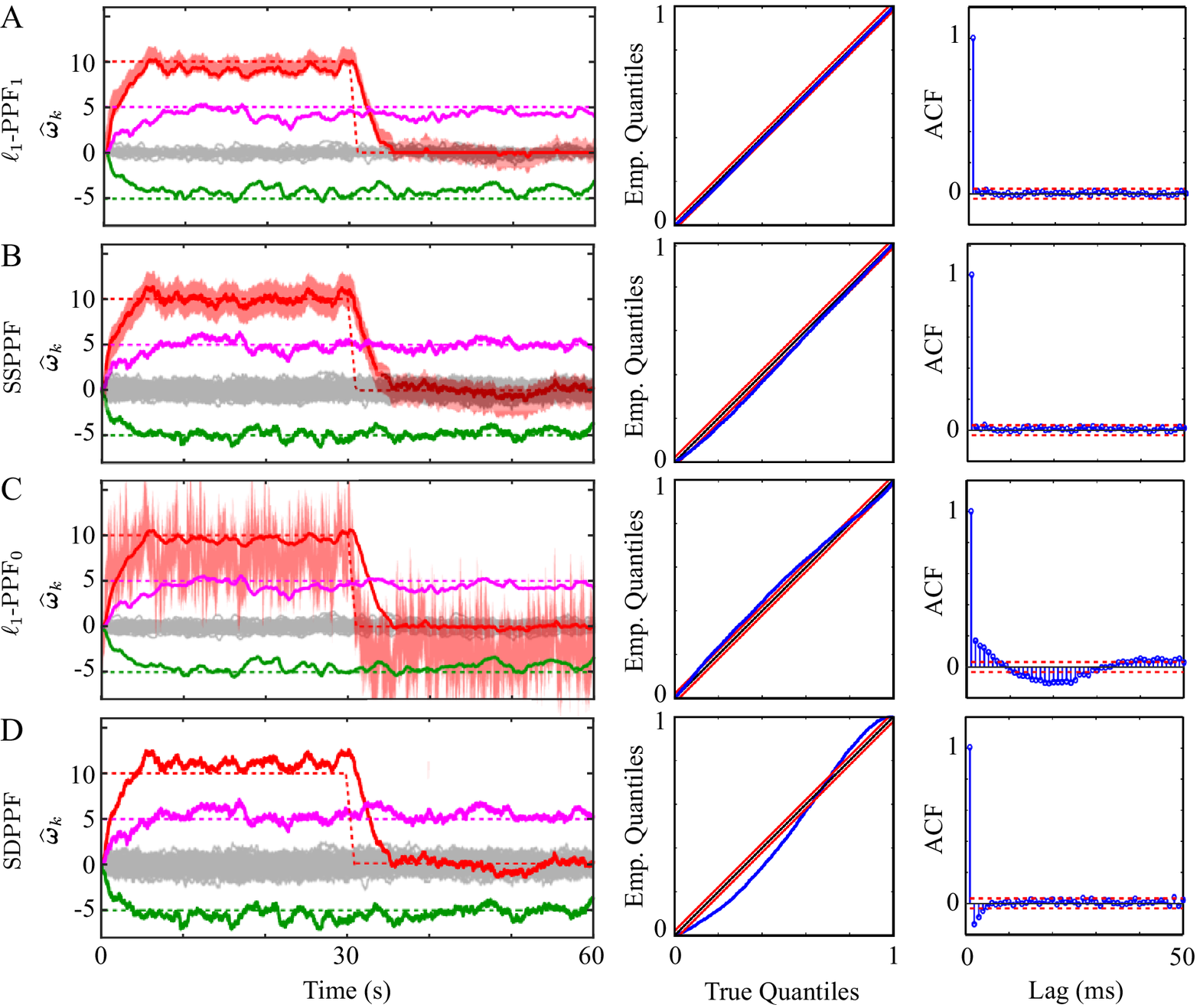}
\vspace{-2mm}
\caption{Performance comparison of the adaptive filtering algorithms: A) $\ell_1$-PPF$_1$, B) SSPPF, C) $\ell_1$-PPF$_0$, and D) SDPPF. In each row, the left panel shows the trie parameter vector with dashed traces and the estimates with solid traces. Colored hulls show the $95\%$ confidence intervals for one of the components. The middle and right panels show the corresponding KS and ACF test plots, respectively. Red traces show confidence regions at a level of $95\%$ for both tests.}
\label{fig:f3}
\vspace{-4mm}
\end{figure*}

In the second simulation scenario, we consider a more realistic setting where $\boldsymbol{\omega}_k$ evolves in time. Furthermore, as in the case of real data applications, we assume that the support of $\boldsymbol{\omega}_k$ is not available as a performance benchmark and resort to statistical goodness-of-fit test. These tests for point process models have been developed as an application of the time-rescaling theorem \cite{brown2002time, haslinger2010discrete} and consist of the Kolmogorov-Smirnov (KS) test for assessing the conditional intensity estimation accuracy, and the Autocovariance Function (ACF) test to assess the conditional independence assumption. We skip the details in the interest of space, and refer the readers to the aforementioned references for a detailed treatment. 

As in the previous case, we consider a bin size of $\Delta = 1 ms$, window size of $W=1$, and a total observation window of $\mathcal{T} = 60 sec$ ($K = 60000$ bins). The stimulus is generated as in the previous case. For the parameter vector $\boldsymbol{\omega}_k$, we choose a fixed baseline rate of $\mu_k = -2.51$ to set the baseline spiking rate to $\lambda_{avg} \Delta \approx 0.1$, and select a sparse modulation vector $\boldsymbol{\theta}_k$ of length $M = 100$ with a support $S = \{1, 10, 20 \}$ of size $L=3$, and respective values of $(\boldsymbol{\theta}_k)_{\{1,10,20\}} = \{10, -5, 5\}$ for $k \le K/2$. Halfway through the test, at $k = K/2+1$, the largest component $(\boldsymbol{\theta}_k)_{1}$, drops  rapidly and linearly to $0$, within a window of length $1~sec$ and remains zero for the rest of the run.

Figure \ref{fig:f3} shows the performance of all four algorithms in the aforementioned setting. Each row (A through D) shows the true time-varying parameter vector (dashed traces) as well as the filtered estimates (solid traces) in the left panel. In particular, the gray solid traces show the out-of-support components which must ideally be equal to zero. The colored hulls around $(\widehat{\boldsymbol{\theta}}_k)_1$ show the $95\%$ confidence intervals (note that confidence intervals for SDPPF cannot be directly obtained and require averaging over multiple realizations). The middle and right panels show the KS and ACF test results at a $95\%$ confidence, respectively. For the quadratic algorithms $\ell_1$-PPF$_1$ and SSPPF, a forgetting factor of $\beta = 0.9995$ is chosen. The regularization parameter for $\ell_1$-PPF$_1$ is chosen as $\gamma = 0.5$, obtained by the aforementioned two-fold even-odd cross validation. For the first order algorithm $\ell_1$-PPF$_0$, a smaller forgetting factor of $\beta = 0.995$ is chosen to ensure stability, and a value of $\gamma = 0.1$ is used based on cross validation. These settings ensure that all the algorithms are tuned in their optimal operating point for fairness of comparison.

Figure \ref{fig:f3}--A and \ref{fig:f3}--B reveal three striking performance gaps between the two second-order algorithms (with the same computational complexity, quadratic in $M$): first, the out-of-support components (gray traces) of $\ell_1$-PPF$_1$ are significantly smaller than those of SSPPF; second, the confidence regions of $\ell_1$-PPF$_1$ are narrower than those of SSPPF; and third, $\ell_1$-PPF$_1$ fully passes the KS test, while SSPPF marginally does so. Similarly, comparing the two first order algorithms (with the same computational complexity, linear in $M$) Figure \ref{fig:f3}--C and \ref{fig:f3}--D reveal that the $\ell_1$-PPF$_0$ significantly suppresses the out-of-support components as compared to SDPPF. Moreover, $\ell_1$-PPF$_0$ provides confidence bounds, which cannot be directly obtained for SDPPF. Finally, $\ell_1$-PPF$_0$ marginally fails the KS test, whereas SDPPF does so significantly. Both algorithms fail the ACF test, which shows that the second-order corrections embedded in $\ell_1$-PPF$_1$ and SSPPF is necessary to achieve a better goodness-of-fit, which a price of higher computational complexity. 

We also inspect the estimated firing probability $\lambda_k(\widehat{\boldsymbol{\omega}}_k) \Delta$ for the four algorithms in Figure \ref{fig:f4}. In addition, we include the probability estimated by the normalized reverse correlation (NRC) method, which is commonly used in neural data analysis, and fits the modulation parameters using a linear model. Figure \ref{fig:f4} shows the true spiking probability (blue solid trace) and the resulting spikes (black vertical lines). In the subsequent rows (B through F), the true and estimated probabilities are shown by dashed blue and solid red traces, respectively. A comparison of all the rows reveals that $\ell_1$-PPF$_1$ and $\ell_1$-PPF$_0$ outperform SSPPF and SDPPF, respectively, in terms of estimating the true probability. The NRC method is inferior to the preceding four algorithms, and results in negative estimates of the probability due to its use of a linear model (as opposed to logistic).

\begin{figure}
\centering
\includegraphics[width = .98\columnwidth]{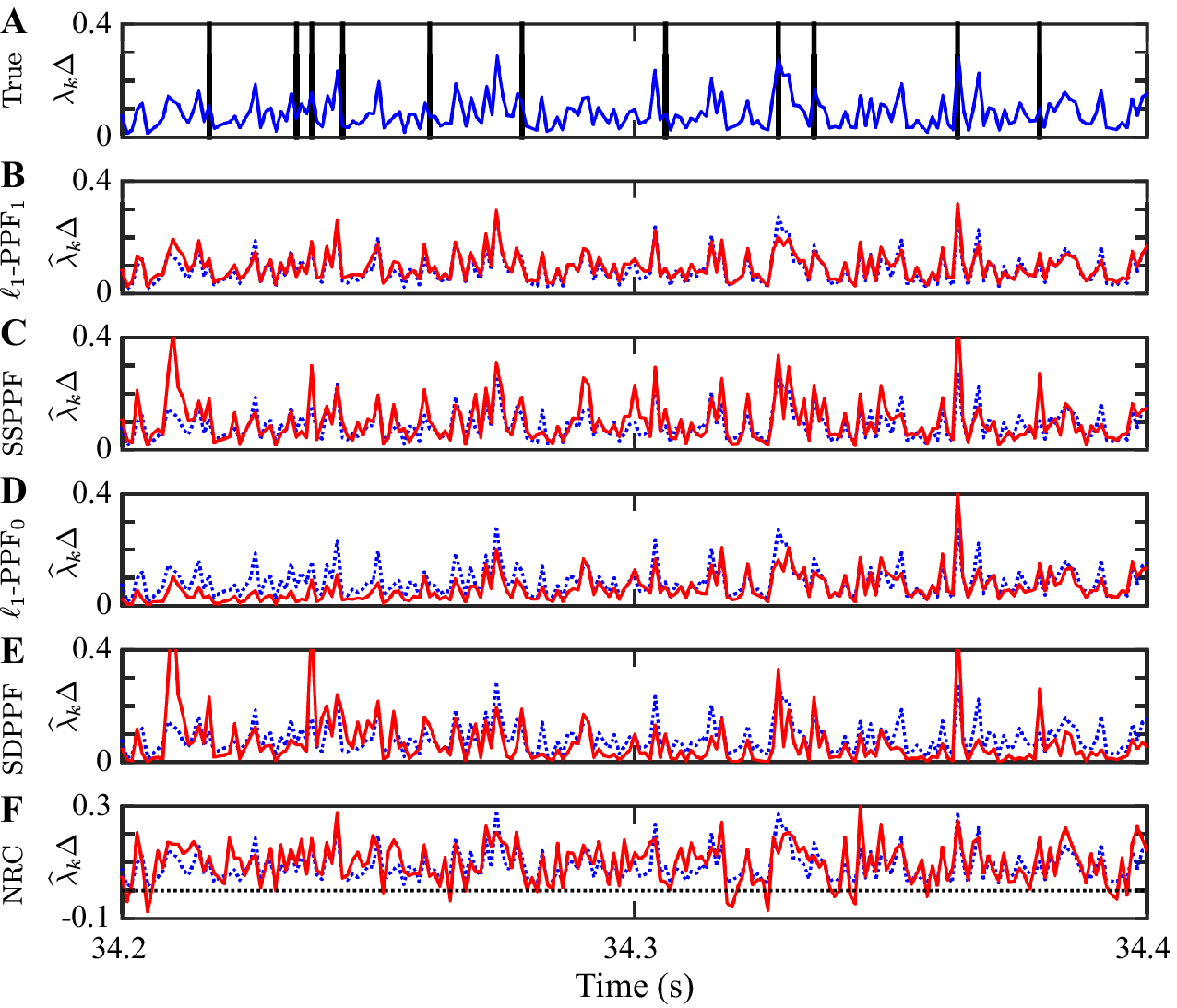}
\vspace{-3mm}
\caption{Firing rate estimates for adaptive filtering algorithms within an interval of $34.2~s \le t \le 34.4~s$: A) true rate (blue solid trace) and spikes (black vertical lines), B) $\ell_1$-PPF$_1$, C) SSPPF, D) $\ell_1$-PPF$_0$, E) SDPPF, and F) normalized reverse correlation (NRC). In rows B through F, the dashed blue traces and solid red traces show the true rate and the estimated rate, respectively.}
\label{fig:f4}
\vspace{-2mm}
\end{figure}

\vspace{-3.5mm}
\subsection{Application to Real Data: Dynamic Analysis of Spectrotemporal Receptive Field Plasticity}

The responses of neurons in the primary auditory cortex (A1) can be characterized by their spectrotemporal receptive fields (STRFs), where each neuron is tuned to a specific region in the time-frequency plane, and only significantly spikes when the acoustic stimulus contains spectrotemporal contents matching its tuning region \cite{depireux2001spectro} (See Figure \ref{fig:f5}, top row, leftmost panel). Several experimental studies have revealed that receptive fields undergo rapid changes in their characteristics during attentive behavior in order to capture salient stimulus modulations  \cite{fritz2003rapid, fritz2005active, mesgarani2010computational}. In \cite{fritz2003rapid}, it is suggested that this rapid plasticity has a significant role in the functional processes underlying active listening. However, most of the widely-used estimation techniques (e.g., normalized reverse correlation) provide static estimates of the receptive field with a a temporal resolution of the order of minutes. Moreover, they do not systematically capture the inherent sparsity manifested in the receptive field characteristics.

\begin{figure*}
\centering
\includegraphics[width = \textwidth]{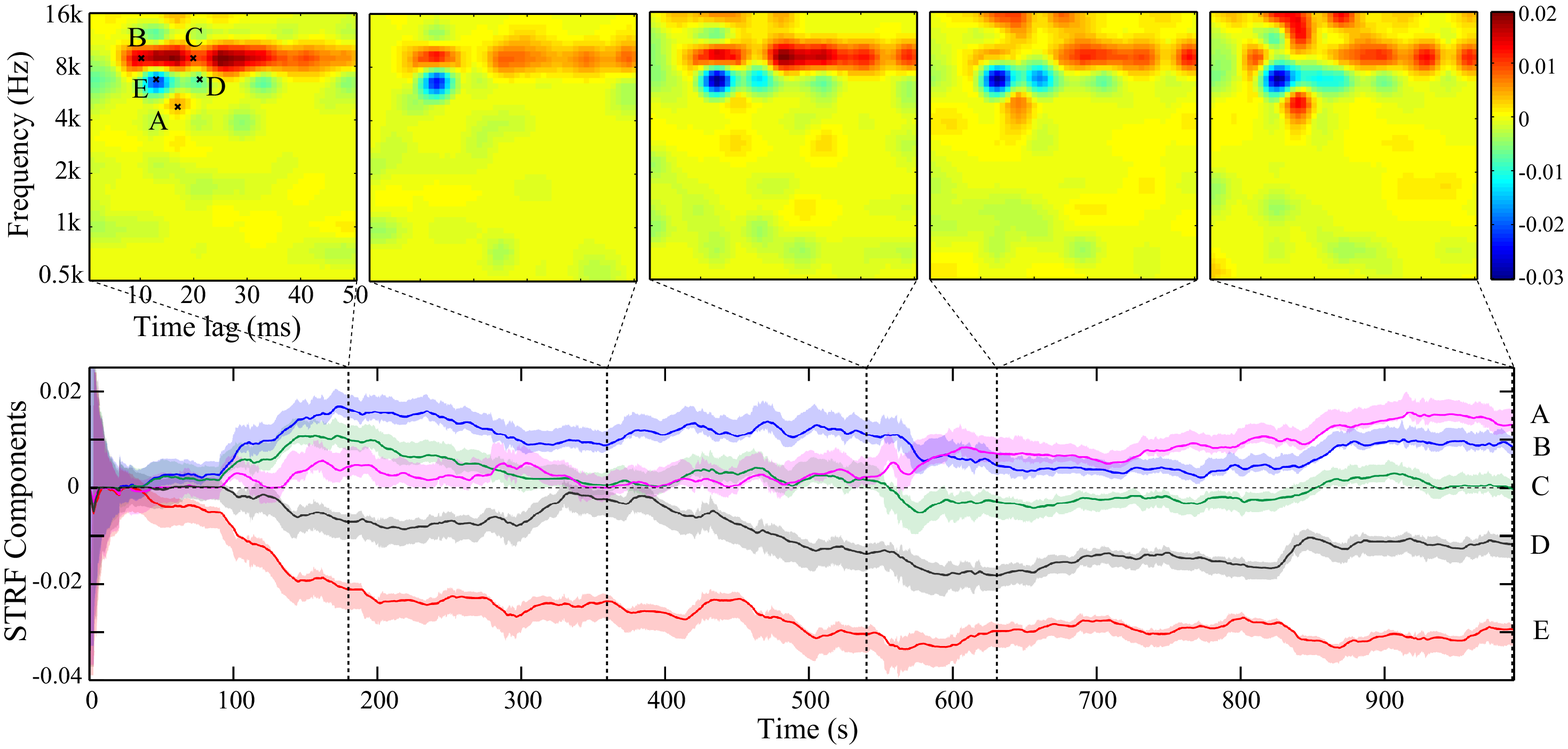}
\vspace{-7mm}
\caption{The time-course of task-dependent STRF plasticity of a ferret A1 neuron. The top row shows snapshots of the STRF at five selected points in time, marked by the dashed vertical lines in the bottom graph. The bottom graph shows the time-course of five selected points (A through E) in the STRF marked on the leftmost panel of the top row.}
\label{fig:f5}
\vspace{-6mm}
\end{figure*}

In the context of our model, the STRF can be modeled as an $(I \times J)$-dimensional matrix, where $I$ and $J$ denote the number of time lags and frequency bands, respectively. By vectorizing this matrix, we obtain an $(M-1)$-dimensional vector $\boldsymbol{\theta}_k$ at window $k$, where $M= I \times J + 1$. Augmenting the baseline rate parameter $\mu_k$, we can model the activity of the A1 neurons using the logistic CIF with a parameter $\boldsymbol{\omega}_k := [ \mu_k, \boldsymbol{\theta}_k]'$. The stimulus vector at time $t$, $\mathbf{s}_t$ is given by the vectorized version of the spectrogram of the acoustic stimulus with $J$ frequency bands and $I$ lags. In order to capture the sparsity of the STRF in the time-frequency plane, we further represent $\boldsymbol{\theta}_k$ over a Gabor time-frequency dictionary consisting of Gaussian windows centered around a regular subset of the $I \times J$ time-frequency plane. That is, for $\boldsymbol{\theta}_k = \mathbf{F} \boldsymbol{\xi}_k$, where $\mathbf{F}$ is the dictionary matrix and $\boldsymbol{\xi}_k$ is the sparse representation of the STRF. The estimation procedures of this paper can be applied to $\boldsymbol{\xi}_k$, by absorbing the dictionary matrix into the data matrix $\mathbf{X}_k$ at window $k$.

We apply our proposed adaptive filter $\ell_1$-PPF$_1$ to multi-unit spike recordings from the ferrets A1 during a series of passive listening conditions and active auditory task conditions. During each active task, ferrets attended to the temporal dynamics of the sounds, and discriminated the rate of acoustic clicks \cite{fritz2005active}. The STRFs were estimated from the passive condition, where the quiescent animal listened to a series of broadband noise-like acoustic stimuli known as Temporally Orthogonal Ripple Combinations (TORC). The experiment consisted of 2 active and 11 passive blocks. Within each passive block, 30 TORCs were randomly repeated a total of 4-5 times each. In our analysis, we pool the spiking data corresponding to the same repeated TORC within each block. Therefore, the time axis corresponds to the experiment time modulo repetitions within each block. We discretize the resulting duration of $T = 990s$ to time bins of size $\Delta = 1~ms$, and segment data to windows of size $W = 10$ samples ($10~ms$). The STRF dimensions are $50 \times 50$, regularly spanning lags of 1 to $50~ms$ and frequency bands of $0.5~kHz$ to $16~kHz$ (in logarithmic scale). The dictionary $\mathbf{F}$ consists of $13 \times 13$ Gabor atoms, evenly spaced within the STRF domain. Each atom is a two-dimensional Gaussian kernel with a variance of $D^2/4$ per dimension, where $D$ denotes the spacing between the atoms. We selected a forgetting factor of $\beta = 0.9998$, a step size of $\alpha = \frac{4(1-\beta)}{M W \bar{\sigma}^2}$, where $\bar{\sigma}^2$ is the average variance of the spectrogram components, $R = 1$ iterations per sample, and a regularization parameter of $\gamma = 40$ via two-fold even-odd cross validation.

Figure \ref{fig:f5}, top row, depicts five snapshots taken at $\{180, 360, 540, 630, 990\}~sec$ corresponding to the end-points of the $\{2, 4, 6, 7, 11\}$th passive tasks. The bottom row shows the time-course of five selected points (marked as A through D in the leftmost panel of the top row) of the STRF during the experiment. The STRF snapshots at times $180$ and $540~sec$ correspond to $90~secs$ after the two active tasks, respectively, and verify the sharpening effect of the excitatory region ($\sim30~msec, 8~kHz$) due to the animal's attentive behavior following the active task reported in \cite{fritz2003rapid}. Moreover, the STRF snapshots at times $360$ and $630~sec$ reveal the weakening of the excitatory region long after the active task and returning to the pre-active state, highlighting the plasticity of A1 neurons. Previous studies have revealed the STRF dynamics with a resolution of the order of minutes \cite{mesgarani2010computational}. Our result in Figure \ref{fig:f5} provides a temporal resolution of the order of centiseconds (3 orders of magnitude increase), while capturing the STRF sparsity in a robust fashion.

\vspace{-2mm}
\section{Concluding Remarks}\label{sec:conclusion}
In this paper, we considered recursive estimation of the time-varying parameter vectors in a logistic regression model for binary time series driven by continuous input. To this end, we integrated several techniques from compressed sensing, adaptive filtering, optimization and statistics. We constructed an objective function which enjoys from the trackability features of the RLS-type algorithms, sparsifying features of $\ell_1$-minimization, and unlike the rate-based linear models commonly used to analyze spiking data, takes into account the binary statistics of the observations. We analyzed the maximizers of the objective function in a rigorous fashion, revealing novel trade-offs between various model parameters. We constructed two adaptive filters, with respective linear and quadratic complexity requirements, for recursive maximization of the objective function in an online setting. Moreover, we characterized the statistical confidence regions for our estimates, and devised a recursive procedure to compute them efficiently. 

Although we specialized our treatment to logistic statistics and $\ell_1$-regularization, our approach to algorithm development has a plug-and-play feature: other GLM link functions (e.g., log-link) with possibly history dependent covariates and other regularization schemes (e.g., re-weighted $\ell_1$, or group-sparse regularization) can be used instead and result in a large class of adaptive filters for sparse point process regression. We tested the performance of our algorithms on simulated as well as experimentally recorded spiking data. Our simulation studies revealed that the proposed filters outperform several existing point process filters. Application of our filters to real data from the ferret primary auditory cortex provided a high-resolution characterization of the time-course of spectrotemporal receptive field plasticity, with 3 orders of magnitudes increase in temporal resolution. Although we focused on auditory neurons, we expect a similar superior performance of our filters when applied to other sensory or motor neurons (e.g., cells in primary or supplementary motor cortex \cite{gandolfo2000cortical}).

\appendices
\vspace{-2mm}
\section{Proof of Theorem \ref{thm:main}}\label{proof}

The proof is mainly based on the beautiful treatment of Negahban et al. \cite{Negahban}. The major difficulty in our case lies in the high inter-dependence of the covariates, which form a Toeplitz structure due to the setup of adaptive filtering. We address the latter issue by adopting techniques from another remarkable paper by Haupt et al. \cite{Toeplitz} to deal with the underlying interdependence. In the process, we also employ concentration inequalities for dependent random variables due to van de Geer \cite{van_de_geer}.

In order to proceed, we adopt the notion of Strong Restricted Convexity (RSC) introduced in \cite{Negahban}. For a twice differentiable log-likelihood with respect to $\boldsymbol{\omega}$, the RSC property or order $L$ implies the existence of a lower quadratic bound on the negative log-likelihood:
\begin{equation}
\label{RSC}
\nonumber \mathcal{D_L}(\boldsymbol{\Delta},\boldsymbol{\omega}):= -\mathcal{L}^\beta(\boldsymbol{\omega}+\boldsymbol{\Delta})+  \mathcal{L}^\beta(\boldsymbol{\omega})+\boldsymbol{\Delta}'\nabla\mathcal{L}^\beta(\boldsymbol{\omega})\geq \kappa\|\boldsymbol{\Delta}\|_2^2,
\end{equation}
for a positive constant $\kappa > 0$ and all $\boldsymbol{\Delta}\in\mathbb{R}^M$ satisfying:
\begin{equation}
\label{eq:cone}
\| \boldsymbol{\Delta}_{S^c} \|_1 \leq 3 \|\boldsymbol{\Delta}_S\|_1 + 4 \sigma_S(\boldsymbol{\omega}).
\end{equation}
for any index set $S \subset \{1,2,\cdots,M\}$ of cardinality $L$. The following key lemma establishes the RSC for $\mathcal{L}^{\beta}(\boldsymbol{\omega})$:
\medskip

\begin{lemma}
\textit{
Let $\{\mathbf{x}_t\}_{t=1}^{KW}$ denote a sequence of covariates and let $\boldsymbol{\omega}$ denote the corresponding logistic parameters. Then, for a fixed positive constant $d > 0$, there exist constants $C'$ and $\kappa > 0$ such that for $\beta \ge 1 - \frac{C'}{L^2 \log M}$ and $K \ge \frac{\log 2}{\log\left( \frac{1}{\beta}\right)}$ the negative log-likelihood $-\mathcal{L}^{\beta}(\boldsymbol{\omega})$ satisfies the RSC of order $L$ with constant $\frac{\kappa}{1-\beta}$ with probability greater than $1-\frac{3}{M^d}$. The constants $C'$ and $\kappa$ are only functions of $d$, $p_{\min}$, $p_{\max}$, $\sigma^2$, $B$, $W$, and are explicitly given in the proof.}
\end{lemma}
\medskip

\begin{proof}
The proof is inspired by the elegant treatment of Negahban et al. \cite{Negahban}. The major difficulty in our setting is the high interdependence of successive covariates due to the shift structure induced by the adaptive setting, whereas in \cite{Negahban}, the matrix of covariates is composed of i.i.d. rows. Using the Taylor's theorem, $\mathcal{D_L}(\boldsymbol{\Delta},\boldsymbol{\omega})$ can be written as:
\begin{equation}
\sum_{i=1}^K \sum_{j=1}^W  \beta^{K-i} \frac{\exp\big(\mathbf{x}'_{(i-1)W+j} \boldsymbol{\omega}^{\star}\big) |\boldsymbol{\Delta}' \mathbf{x}_{(i-1)W+j}|^2}{\left (1 + \exp\big(\mathbf{x}'_{(i-1)W+j} \boldsymbol{\omega}^{\star}\big) \right )^2},
\end{equation}
with  $\boldsymbol{\omega}^{\star} = \boldsymbol{\omega} + \tau \boldsymbol{\Delta}$ for some $\tau \in (0,1)$. Since by hypothesis $0 < p_{\min} \le \lambda_i \Delta \le p_{\max} < 1$, we have:
\begin{equation}
\textstyle \frac{\exp\big(\mathbf{x}'_{(i-1)W+j} \boldsymbol{\omega}^{\star}\big)}{\left (1 + \exp\big(\mathbf{x}'_{(i-1)W+j} \boldsymbol{\omega}^{\star}\big) \right )^2} \ge p_{\min} (1-p_{\max}).
\end{equation}
We can therefore further lower bound $\mathcal{D_L}(\boldsymbol{\Delta},\boldsymbol{\omega})$ by:
\begin{equation}\label{eq:bound}
\mathcal{D_L}(\boldsymbol{\Delta},\boldsymbol{\omega}) \ge  p_{\min} (1 - p_{\max}) \sigma^2 N_\beta  \left \{ \boldsymbol{\Delta}' \mathbf{C}_\beta \boldsymbol{\Delta} \right \},
\end{equation}
where $N_\beta := W \frac{1 - \beta^{K+1}}{1- \beta}$, and
\begin{equation}
\mathbf{C}_\beta :=  \frac{1}{\sigma^2 N_\beta} \sum_{i=1}^K \sum_{j=1}^W \beta^{K-i} {\mathbf{x}_{(i-1)W+j} \mathbf{x}'_{(i-1)W+j}}.
\end{equation}
Note that the matrix $\mathbf{C}_\beta$ has highly inter-dependent elements due to the Toeplitz structure in the adaptive design. In order to establish the RSC condition, we show the stronger Restricted Eigenvalue (RE) property, which in turn implies RSC \cite{bickel2009simultaneous}. Let $\delta \in (0,1)$ be fixed. To do so, we need to bound the eigenvalues of $(\mathbf{C}_\beta)_S$, the restriction of $\mathbf{C}_\beta$ to a subset of columns and rows corresponding to indices in $S \subset \{1,2,\cdots,M\}$ with $|S| = rL$, for some integer $r > 1 + \frac{9(1+\delta)}{1-\delta}$.

Without loss of generality, we assume that the first entry of the covariate vectors $\mathbf{x}_t$ is replaced by $\sigma$ instead of $1$, for presentational simplicity of the following treatment. For $m, m' \neq 1$, we have: 
\begin{align}
\nonumber (\mathbf{C}_\beta)_{m,m'} = \frac{1}{\sigma^2 N_\beta} \sum_{i=1}^K \sum_{j=0}^{W-1} \beta^{K-i} & s_{(i-1)W + j + m - M} \\
\nonumber & \times s_{(i-1)W + j + m' - M}.
\end{align}
For $m = m' = 1$, 
\begin{equation}
\nonumber (\mathbf{C}_\beta)_{1,1} = \frac{1}{\sigma^2 N_\beta} \sum_{i=1}^K  W \beta^{K-i} \sigma^2= \frac{1}{N_\beta} W\frac{1- \beta^{K+1}}{1-\beta} = 1,
\end{equation}
and for $m \neq 1$,
\begin{equation}
\nonumber (\mathbf{C}_\beta)_{m,1} = (\mathbf{C}_\beta)_{1,m} = \frac{1}{\sigma N_\beta} \sum_{i=1}^K \sum_{j=0}^{W-1} \beta^{K-i} s_{(i-1)W + j + m - M}. 
\end{equation}
We also have $\mathbb{E} \{ (\mathbf{C}_\beta)_{m,m'} \} = \delta_{mm'}$. Using Hoeffding's inequality \cite{hoeffding1963probability} we get:
\begin{align}
\nonumber \mathbb{P}\left ( \left | (\mathbf{C}_\beta)_{m,m} - 1 \right| > t \right ) & \le 2 \exp \left ( - \frac{2 t^2 N^2_\beta}{B^4\sum_{i=1}^K W \beta^{2(K-i)}}\right )\\
\nonumber & = 2 \exp \left ( - \frac{2 N^2_\beta t^2 \sigma^4}{B^4 N_{\beta^2}}\right )\\
& \le 2 \exp \left ( - \frac{2 N_\beta t^2 \sigma^4}{B^4}\right ), 
\end{align}
since $N_{\beta^2} = N_\beta \frac{1+\beta^{K+1}}{1+\beta} \le N_{\beta}$, for $\beta \in [0,1]$.
Similarly,
\begin{align}
\nonumber \mathbb{P}\left ( \left | (\mathbf{C}_\beta)_{1,m} \right| > t \right ) \le 2 \exp \left ( - \frac{2 N_\beta t^2 \sigma^2}{B^2}\right ),
\end{align}
By adopting the partitioning technique of Theorem 4 in \cite{Toeplitz}, we also get for $m \neq m'$:
\begin{align}
\nonumber \mathbb{P}\left ( \left | (\mathbf{C}_\beta)_{m,m'} \right| > t \right ) \le 4 \exp \left ( - \frac{N_\beta t^2 \sigma^4}{8B^4}\right ).
\end{align}
Let $B_0 := \max \{ B^2, \frac{B^4}{\sigma^2} \}$. Now, using the union bound we have:
\begin{align}
\nonumber \mathbb{P}\Bigg (\! \bigcup_{\underset{m \neq m'}{m,m'=1}}^M  \!\!\bigg \{ \big | &(\mathbf{C}_\beta)_{m,m'} \big | \! >\! \frac{\delta}{2rL} \bigg \} \!\Bigg ) \!\!\le 2 M^2 \!\exp \left (\! - \frac{N_\beta \delta^2 \sigma^2}{32B_0 r^2 L^2}\!\right )\!,
\end{align}
where we have used ${M \choose 2} < \frac{M^2}{2}$ and,
\begin{align}
\nonumber \mathbb{P}\left (\bigcup_{m=1}^M \! \left \{ \left | (\mathbf{C}_\beta)_{m,m} - 1 \right| > \!\frac{\delta}{2rL} \right \} \!\right )\! \le 2 M \exp \left ( \!- \frac{N_\beta \delta^2 \sigma^2}{4B_0 r^2 L^2}\!\right )\!.
\end{align}
Now, by invoking the Gershgorin's disc theorem, the eignevalues of any sub-matrix of $\mathbf{C}_\beta$ restricted to an index set $S$ with $|S| = rL$, lie in the interval $[1-\delta, 1+\delta]$ with probability at least:
\begin{align}
\nonumber & 1 - 2 M^2 \exp \left ( - \frac{N_\beta \delta^2 \sigma^2}{32B_0 r^2 L^2}\right ) - 2 M \exp \left ( - \frac{N_\beta \delta^2 \sigma^2}{4B_0 r^2 L^2}\right )\\
& \ge 1 - 3 M^2 \exp \left ( - \frac{N_\beta \delta^2 \sigma^2}{32B_0 r^2 L^2}\right ).
\end{align}
Hence, by choosing $N_\beta \ge \frac{32  B_0 \sigma^2 r^2(d+2)}{\delta^2} L^2 \log M$, the probability above is greater than $1 - \frac{3}{M^d}$.

Next, by invoking Lemma 4.1 (ii) of \cite{bickel2009simultaneous}, we have that $\mathbf{C}_\beta$ satisfies the RSC condition over the set given by Eq. (\ref{eq:cone}) with a constant given by:
\begin{equation}
\kappa_0 = (1 - \delta) \left ( 1 - 3 \sqrt{\textstyle \frac{1 + \delta}{(r-1)(1-\delta)}}\right )^2.
\end{equation}
Hence, the negative log-likelihood satisfies the RSC with a constant given by $p_{\min} (1- p_{\max}) \sigma^2 N_\beta \kappa_0$. Finally, by taking $K \ge \frac{\log 2}{\log \left (\frac{1}{\beta}\right)}$, we have that $N_\beta \ge \frac{W}{2(1 - \beta)}$, which makes $\kappa$ independent of $K$ and $\beta$, given by:
\begin{equation}
\kappa := \frac{p_{\min} (1-p_{\max}) \sigma^2 \kappa_0 W}{2},
\end{equation}
and $\beta \ge 1 - \frac{C'}{L^2 \log M}$ with $C' := \frac{W\delta^2}{64  B_0 \sigma^2 r^2(d+2)}$.
\end{proof}
\medskip

Next, the result of Theorem 1 of \cite{Negahban} implies:
\begin{equation}
\| \widehat{\boldsymbol{\omega}} - \boldsymbol{\omega} \|_2 \le \frac{2 \gamma \sqrt{L}}{\kappa} + \sqrt{\frac{2 \gamma \sigma_L(\boldsymbol{\omega})}{\kappa}},
\end{equation}
for $\gamma > 2 \| \nabla \mathcal{L}^{\beta}(\boldsymbol{\omega}) \|_\infty$. We have, for $m \neq 1$,
\begin{align}
\nonumber \left ( \nabla \mathcal{L}^{\beta}(\boldsymbol{\omega}) \right )_m &= \sum_{i=1}^K \sum_{j=1}^W \beta^{K-i} {{s}_{(i-1)W+j+m-M+1}}\\
\nonumber & \left \{ n_{(i-1)W + j+m-M+1} - \lambda_{(i-1)W+j+m-M+1} \Delta\right\}.
\end{align}
Now, let $\mathcal{F}_t$ be the $\sigma$-field generated by $s_{-M+1}, \cdots, s_t$, i.e., $\sigma(s_{-M+1},\cdots,s_t)$. We have that
\begin{align}
\nonumber \mathbb{E} \left [ \left \{ n_{t} - \lambda_{t} \Delta\right\} {{s}_{t}}\right ] &= \mathbb{E} \left [ \mathbb{E} \left [ \left \{ n_{t} - \lambda_{t} \Delta\right\} {s}_{t} | \mathcal{F}_{t} \right ] \right ]\\
\nonumber & = \mathbb{E} \left [ s_t \mathbb{E} \left [ \left \{ n_{t} - \lambda_{t} \Delta\right\} | \mathcal{F}_{t} \right ] \right ]\\
\nonumber & = \mathbb{E} \Big [ s_t \mathbb{E} \big [ \underbrace{\left \{ \lambda_{t} \Delta - \lambda_{t} \Delta\right\}}_{=0} | \mathcal{F}_{t} \big ] \Big ] = 0.
\end{align}
Hence for all $m$, $\mathbb{E} \left \{ \left ( \nabla \mathcal{L}^{\beta}(\boldsymbol{\omega}) \right )_m \right \} = 0$. We next invoke the following result for concentration of dependent random variables:
\begin{prop}
\label{hoeff_dep}
Consider a sequence of $\sigma$-fields $\mathcal{F}_0 \subset \mathcal{F}_1 \subset \cdots$. Suppose that $X_i$ is $\mathcal{F}_i$-measurable with $|X_i| \le B_i$ for some constant $B_i$, $i=1,2,\cdots$ and $\mathbb{E} \{X_i | \mathcal{F}_{i-1} \} = 0$. Then for all $t>0$,
\begin{equation*}
\mathbb{P} \left( \sum_{i=1}^n X_i \geq t \right) \leq \exp\left(-\frac{2 t^2}{\sum_{i=1}^n B^2_i }\right).
\end{equation*}
\end{prop}
\begin{proof}
This result is a special case of Theorem 2.5 of \cite{van_de_geer} for \textit{bounded} and possibly dependent random variables, which generalizes Hoeffding's inequality.
\end{proof}
In our case, we can similarly show that $\mathbb{E} \left [ s_t  \left \{ n_{t} - \lambda_{t} \Delta\right\} | \mathcal{F}_{t-1} \right ] = \mathbb{E} \left [ s_t \mathbb{E} \left [ \left \{ n_{t} - \lambda_{t} \Delta\right\} | \mathcal{F}_{t-1}, \mathcal{F}_t \right ] \right ] = 0$. Moreover, each summand is bounded by $2 \beta^{K-i} B$. Hence, using the result of Proposition \ref{hoeff_dep}, by taking $n=TW$ and $X_i =  s_i [n_i - \lambda_i \Delta ]$, we get: 
\begin{align}
\nonumber \mathbb{P} \left ( \left |\left ( \nabla \mathcal{L}^{\beta}(\boldsymbol{\omega}) \right )_m \right | > t N_\beta \right ) &\le 2 \exp \left ( -\frac{t^2 N^2_\beta}{2 \sum_{i=1}^K W \beta^{2(K-i)}}\right )\\
\nonumber & \le 2 \exp \left ( -\frac{N_\beta t^2}{2}\right ).
\end{align}
Using the union bound, we have:
\begin{align}
\mathbb{P} \left ( \left \| \nabla \mathcal{L}^{\beta}(\boldsymbol{\omega}) \right \|_\infty > t N_\beta \right ) \le 2M \exp \left ( -\frac{N_\beta t^2}{2}\right ).
\end{align}
By choosing $t = \sqrt{\frac{2(d+1) \log M}{N_\beta}}$, we have that
\begin{equation}
\left \| \nabla \mathcal{L}^{\beta}(\boldsymbol{\omega}) \right \|_\infty < \sqrt{{2(d+1)N_\beta \log M}},
\end{equation}
with probability at least $1 - \frac{2}{M^d}$.

Hence, for a fixed $\delta < 1$, $d > 0$, and $r > 1 + \frac{9(1+\delta)}{1-\delta}$, by taking $\beta \ge 1 - \frac{C'}{L^2 \log M}$ with $C' := \frac{W\delta^2}{64  B_0 \sigma^2 r^2(d+2)}$ and $\gamma >  C'' \sqrt{\frac{\log M}{1-\beta}}$ with $C'' := \sqrt{4(d+1)W}$, any maximizer $\widehat{\boldsymbol{\omega}}$ satisfies:
\begin{align}
\label{thm2}
\nonumber \left \|\widehat{\boldsymbol{\omega}}\!-\!\boldsymbol{\omega}\right \|_2  \! \leq \!C \sqrt{\!(1\!-\!\beta) L \log M} \!+\! \sqrt{C \sigma_L(\boldsymbol{\omega})}\sqrt[4]{\!(1\!-\!\beta) L \log M},
\end{align}with probability at least $1- \frac{3}{M^d} - \frac{2}{M^d}$, where $C$ is given by
\begin{equation}
\nonumber C := \frac{\sqrt{{128(d+1)}}}{\sqrt{W} p_{\min} (1-p_{\max}) \sigma^2 (1-\delta)\left ( 1 - 3 \sqrt{\frac{1 + \delta}{(r-1)(1-\delta)}}\right )^2}.
\end{equation}

\section{The Proximal Gradient Algorithm}\label{proximal}
In this appendix, we give an overview of the proximal gradient algorithm for minimization of convex function. The corresponding algorithm for maximization of concave functions can be obtained by negating the objective functions. Consider the general optimization problem
\begin{equation}
\min_{\mathbf{x}} \ f(\mathbf{x}) + g(\mathbf{x}), \label{Genprob}
\end{equation}
where functions $f(\mathbf{x}): \mathbb{R}^M \rightarrow \mathbb{R}$ and $g(\mathbf{x}): \mathbb{R}^M \rightarrow \mathbb{R} \cup \{\infty\}$ are assumed to be closed proper convex functions. Suppose that $f$ is differentiable with a Lipschitz continuous gradient $\nabla f$ with constant $L(\nabla f)$. The function $g$ can be possibly non-smooth.
A wide range of practical optimization problems can be cast in this form, particularly in the context of machine learning \cite{scholkopf2001learning}, where the objective function can be decomposed into a loss function and a regularization term.

The proximal gradient algorithm provides an iterative procedure for solving (\ref{Genprob}) in the following form:
\begin{align}
\mathbf{x}^{(\ell+1)} = \mathcal{P}_{\alpha^{(\ell)} g} \Big[\mathbf{x}^{(\ell)} - \alpha^{(\ell)} \nabla f(\mathbf{x}^{(\ell)})\Big] , 
\label{ProxGrad}
\end{align} 
where the parameter $\alpha^{(\ell)}$ is an appropriately chosen step size at iteration $\ell$ so that $\alpha^{(\ell)} < \frac{1}{L(\nabla f)}$, and the \textit{proximal operator} $\mathcal{P}_{\alpha g}(.)$ of function $g$ with parameter $\alpha$ is defined as
\begin{align}
\mathcal{P}_{\alpha g}(\mathbf{x}) := \underset{\mathbf{u}}{\operatorname{argmin}} \ \Big\{g(\mathbf{u}) + \frac{1}{2\alpha} \| \mathbf{u} - \mathbf{x} \|_2^2 \Big\}.
\end{align}
Among the several interpretations available for the proximal gradient method, we have adopt a quadratic approximation-based model to derive the main iterative scheme in (\ref{ProxGrad}). This interpretation \cite{figueiredo2005bound, figueiredo2007majorization}, is based on the Majorization-Minimization algorithm (see \cite{hunter2004tutorial} for a detailed discussion). In the approximation-based derivation, the $\ell$-th iteration for solving the general problem (\ref{Genprob}) can be written in the following form: 
\begin{align}
\mathbf{x}^{(\ell+1)} = \underset{\mathbf{u}}{\operatorname{argmin}} \  {\widehat{f}_\alpha} (\mathbf{x},\mathbf{x}^{(\ell)}) + g(\mathbf{x}), \label{ApproxProb}
\end{align}
where the original objective function $f$ is replaced with a quadratically-regularized linear approximation around the previous iterate $\mathbf{x}^{(\ell)}$, given by
\begin{equation}
\widehat{f}_\alpha(\mathbf{x},\mathbf{y}) := f(\mathbf{y}) + \nabla f(\mathbf{y})'(\mathbf{x}-\mathbf{y}) + \frac{1}{2\alpha} \, \|\mathbf{x}-\mathbf{y} \|_2^2,
\end{equation}
where the quadratic term is referred to as the \textit{trust region penalty}. Modulo constants, the objective function in (\ref{ApproxProb}) can be rearranged to get the proximal gradient form
\begin{align}
\nonumber \mathbf{x}^{(\ell+1)} &= \underset{\mathbf{x}}{\operatorname{argmin}} \Big\{ g(\mathbf{x}) + \frac{1}{2\alpha} \, \| \mathbf{x} - \mathbf{x}^{(\ell)} + \alpha \nabla f(\mathbf{x}^{(\ell)}) \|_2^2 \Big\} \\ 
&= \mathcal{P}_{\alpha g} \Big[\mathbf{x}^{(\ell)} - \alpha \nabla f(\mathbf{x}^{(\ell)})\Big] . 
\end{align}

The proximal operator often admits closed form expressions. As for $\ell_1$-regularization, the proximal operator takes the simple form of the \textit{soft thresholding shrinkage operator} $ \mathcal{P}_{\alpha \|.\|_1} =: \mathcal{S}_{\alpha}$ whose $i$th component is given by
\begin{align}
(\mathcal{S}_{\alpha}(x))_i  :=  \operatorname{sgn}(x_i)(x_i - \alpha)_{+},
\end{align}
with $\operatorname{sgn}$ denoting the standard signum function, and $(a)_+ := \max \{a,0\}$. In this case, the proximal algorithm leads to a family of algorithms referred to as iterative shrinkage algorithms \cite{figueiredo2003algorithm, bruckstein2009sparse,daubechies2004iterative}, where each iteration involves a simple gradient descent step followed by a shrinkage operation. 

Finally, in our setting, the function $f$ is taken to be the exponentially weighted log-likelihood $\mathcal{L}^\beta(\cdot)$. Due to the smoothness of the logistic function, the Lipschitz constant for $\nabla \mathcal{L}^{\beta}(\boldsymbol{\omega}_k)$ can be upper bounded by the trace of the Hessian $\mathbf{B}_k(\boldsymbol{\omega}_k)$ given in Eq. (\ref{eq:hess}). Noting that the elements of $\boldsymbol{\Lambda}_i$ are at most equal to $1/4$, we get $L(\nabla \mathcal{L}^{\beta}(\boldsymbol{\omega}_k)) \le \frac{1}{4} \sum_{i=1}^k \sum_{j=1}^W \beta^{k-i} x^2_{(i-1)W+j}$. Using assumption (1) of Section \ref{maintheorem} and an application of Hoeffding's inequality, shows that the sum is concentrated around its mean given by $\frac{M W \sigma^2}{4(1-\beta)}$, for large enough $k$. Therefore, we choose the step size $\alpha = \frac{(1-\beta)}{c M W \sigma^2}$, for some constant $c \ge 1/4$.

\section{Computation of Confidence Intervals}\label{app:conf}
The $\ell_1$-regularized ML estimate of Eq. (\ref{MainProb}) can be written in the following form
\vspace{-1mm}
\begin{align}
\widehat{\boldsymbol{\omega}}_k = \underset{\boldsymbol{\omega}_k}{\operatorname{argmax}} \left \{ \mathfrak{P}_{\beta} \mathcal{L}(\boldsymbol{\omega}_k)  - \gamma \|  \boldsymbol{\omega}_k\|_1 \right\},
\label{MainML}
\end{align}
where $\mathcal{L}(\boldsymbol{\omega}) := \log p\big(\mathbf{n} | \mathbf{X} , \boldsymbol{\omega} \big)$ denotes the log-likelihood function over a generic window with spiking vector $\mathbf{n}$, data matrix $\mathbf{X}$ and parameter vector $\boldsymbol{\omega}$, and the operator $\mathfrak{P}_{\beta} f(\mathbf{n}, \mathbf{X}, \boldsymbol{\omega})$ is defined for a function $f: \{0,1\}^W \times \mathbb{R}^{W\times M} \times \mathbb{R}^M \rightarrow \mathbb{R} $ as the empirical expectation exponentially weighted with a forgetting factor $\beta$:
\vspace{-1mm}
\begin{align}
\mathfrak{P}_{\beta} f(\boldsymbol{\omega}_k) = \sum_{i=1}^k \beta^{k-i} f\big(\mathbf{n}_i,\mathbf{X}_i; \boldsymbol{\omega}_k \big),
\end{align}
where we have suppressed the dependence of $f$ on $\mathbf{n}$ and $\mathbf{X}$ on the left hand side for notational simplicity. Following the treatment of Theorem 3.1 of \cite{van_de_geer}, the corresponding empirical gradient vector and Hessian are respectively given by:
\vspace{-1mm}
\begin{align}
\mathbf{g}_k(\boldsymbol{\omega}_k) &:= \mathfrak{P}_{\beta} \nabla \mathcal{L} (\boldsymbol{\omega}_k) = \sum_{i = 1}^k \beta^{k-i} \mathbf{X}_i' \boldsymbol{\varepsilon}_i(\boldsymbol{\omega}_k), \\
\mathbf{B}_k(\boldsymbol{\omega}_k) &:= \mathfrak{P}_{\beta} \nabla^2 \mathcal{L}(\boldsymbol{\omega}_k)  =  - \sum_{i = 1}^k \beta^{k-i} \, \mathbf{X}_i' \boldsymbol{\Lambda}_i(\boldsymbol{\omega}_k) \mathbf{X}_i.
\end{align}
The KKT conditions for the estimator $\widehat{\boldsymbol{\omega}}_k$ can be then written as:
\vspace{-2mm}
\begin{align}
\mathbf{g}_k(\widehat{\boldsymbol{\omega}}_k) - \gamma \widehat{\mathbf{s}}_k &= 0, \quad \| \widehat{\mathbf{s}}_k \|_{\infty} \leq 1.
\end{align}
where $\widehat{\mathbf{s}}_k  \in \partial \| \widehat{\boldsymbol{\omega}}_k\|_1$ is a subgradient vector from the subdifferential of the $\ell_1$ norm, with components $(\widehat{\mathbf{s}}_{k})_m = \operatorname{sgn}\left  ( (\widehat{\boldsymbol{\omega}}_{k})_m\right )$ for $(\widehat{\boldsymbol{\omega}}_{k})_m \neq 0$ and $| (\widehat{\mathbf{s}}_{k})_m | \leq 1$ otherwise, for $m=1,2,\dots,M$. Substituting $\mathfrak{P}_{\beta} \mathcal{L}(\boldsymbol{\omega}_k)$ by its quadratic approximation around the true parameter vector $\boldsymbol{\omega}_k$, and inverting the corresponding KKT conditions, the 'de-sparsified' estimator $\widehat{\mathbf{w}}_k$ can be obtained as:
\vspace{-1mm}
\begin{align}
\widehat{\mathbf{w}}_k :=  \widehat{\boldsymbol{\omega}}_k - \widehat{\boldsymbol{\Theta}}_k  \mathbf{g}_k(\widehat{\boldsymbol{\omega}}_k),\label{despars}
\end{align}
where the matrix $\widehat{\boldsymbol{\Theta}}_k$ is the approximate inverse of Hessian matrix $\mathbf{B}_k(\widehat{\boldsymbol{\omega}}_k) $, and can be computed using the following node wise regression procedure \cite{van_de_geer}. To compute the $m$th row of $\widehat{\boldsymbol{\Theta}}_{k}$, first the solution to the following LASSO problem is obtained:
\vspace{-1mm}
\begin{align}
\nonumber \resizebox{\linewidth}{!}{$\widehat{\boldsymbol{\psi}}_{m} :=  \underset{\boldsymbol{\psi} \in \mathbb{R}^{M-1}}{\operatorname{argmin}} \Big( - 2(\mathbf{B}_{k})_{m,\!\backslash m} \boldsymbol{\psi} + \boldsymbol{\psi}' (\mathbf{B}_{k})_{\backslash m,\! \backslash m} \boldsymbol{\psi} + 2 \gamma_m \|\boldsymbol{\psi} \|_1 \Big)$}, \label{L1Quad}
\end{align}   
where the dependence of $\mathbf{B}_k$ on $\widehat{\boldsymbol{\omega}}_k$ is suppressed for notational convenience, and the subscript notations are the same as those described in the footnote of Algorithm \ref{alg3}. Then, we define the vector $\mathbf{c} \in \mathbb{R}^M$ as:
\vspace{-1mm}
\begin{equation}
(\mathbf{c})_m = 1, \quad (\mathbf{c})_{\backslash m} = - \widehat{\boldsymbol{\psi}}_m,
\end{equation}
and the scaling constant $\tau^2_m$ as
\vspace{-1mm}
\begin{equation}
 {\tau}^2_{m} := (\mathbf{B}_k)_{m,m} - \widehat{\boldsymbol{\psi}}_{m}^{(L)} (\mathbf{B}_{k})'_{m,\!\backslash m}.
\end{equation}
Finally, the $m$th row of $\widehat{\boldsymbol{\Theta}}_k$ is given by $(\widehat{\boldsymbol{\Theta}}_k)_m := \frac{1}{\tau^2_m} \mathbf{c}$. The variance and the confidence interval at a level of $\alpha$ for the $m$th component of $\widehat{\boldsymbol{\omega}}_{k}$ can then be computed as given in lines 9 and the output of Algorithm \ref{alg3} \cite{van_de_geer}, where
\vspace{-1mm}
\begin{align}
\nonumber \mathbf{G}_k(\boldsymbol{\omega}) \!:= \mathfrak{P}_{\beta^2}\! \nabla\!\mathcal{L} (\boldsymbol{\omega}) \nabla \! \mathcal{L}'(\boldsymbol{\omega}) \!= \! \sum_{i = 1}^k \beta^{2(k-i)} \mathbf{X}_i'\boldsymbol{\varepsilon}_i(\boldsymbol{\omega})\boldsymbol{\varepsilon}_i(\boldsymbol{\omega})'\mathbf{X}_i.
\end{align}
Using Taylor expansion similar to that in the development of $\ell_1$-PPF$_1$, the matrix $\mathbf{G}_k(\widehat{\boldsymbol{\omega}}_k)$ can be recursively updated as given in line 2 of Algorithm \ref{alg3}. Finally, the node wise regression can be recursively computed using the SPARLS algorithm \cite{babadi2010sparls}, which is given in lines 3--5 of Algorithm \ref{alg3}. The parameter $\gamma_m$ can be chosen to be in the same order of $\gamma$ in (\ref{MainProb}).

\vspace{-1.5mm}
{\small
\bibliographystyle{IEEEtran}
\bibliography{ppbib}

\begin{thebibliography}{10}
\providecommand{\url}[1]{#1}
\csname url@samestyle\endcsname
\providecommand{\newblock}{\relax}
\providecommand{\bibinfo}[2]{#2}
\providecommand{\BIBentrySTDinterwordspacing}{\spaceskip=0pt\relax}
\providecommand{\BIBentryALTinterwordstretchfactor}{4}
\providecommand{\BIBentryALTinterwordspacing}{\spaceskip=\fontdimen2\font plus
\BIBentryALTinterwordstretchfactor\fontdimen3\font minus
  \fontdimen4\font\relax}
\providecommand{\BIBforeignlanguage}[2]{{%
\expandafter\ifx\csname l@#1\endcsname\relax
\typeout{** WARNING: IEEEtran.bst: No hyphenation pattern has been}%
\typeout{** loaded for the language `#1'. Using the pattern for}%
\typeout{** the default language instead.}%
\else
\language=\csname l@#1\endcsname
\fi
#2}}
\providecommand{\BIBdecl}{\relax}
\BIBdecl

\bibitem{frank2004hippocampal}
L.~M. Frank, G.~B. Stanley, and E.~N. Brown, ``Hippocampal plasticity across
  multiple days of exposure to novel environments,'' \emph{The Journal of
  neuroscience}, vol.~24, no.~35, pp. 7681--7689, 2004.

\bibitem{depireux2001spectro}
D.~A. Depireux, J.~Z. Simon, D.~J. Klein, and S.~A. Shamma, ``Spectro-temporal
  response field characterization with dynamic ripples in ferret primary
  auditory cortex,'' \emph{Journal of neurophysiology}, vol.~85, no.~3, pp.
  1220--1234, 2001.

\bibitem{Daley}
D.~J. Daley and D.~Vere-Jones, \emph{An introduction to the theory of point
  processes: volume II: general theory and structure}.\hskip 1em plus 0.5em
  minus 0.4em\relax Springer Science \& Business Media, 2007, vol.~2.

\bibitem{ogata1988statistical}
Y.~Ogata, ``Statistical models for earthquake occurrences and residual analysis
  for point processes,'' \emph{Journal of the American Statistical
  Association}, vol.~83, no. 401, pp. 9--27, 1988.

\bibitem{vere1970stochastic}
D.~Vere-Jones, ``Stochastic models for earthquake occurrence,'' \emph{Journal
  of the Royal Statistical Society. Series B}, pp. 1--62, 1970.

\bibitem{brown2004multiple}
E.~N. Brown, R.~E. Kass, and P.~P. Mitra, ``Multiple neural spike train data
  analysis: state-of-the-art and future challenges,'' \emph{Nature
  neuroscience}, vol.~7, no.~5, pp. 456--461, 2004.

\bibitem{brown2001analysis}
E.~N. Brown, D.~P. Nguyen, L.~M. Frank, M.~A. Wilson, and V.~Solo, ``An
  analysis of neural receptive field plasticity by point process adaptive
  filtering,'' \emph{Proceedings of the National Academy of Sciences}, vol.~98,
  no.~21, pp. 12\,261--12\,266, 2001.

\bibitem{smith2003estimating}
A.~Smith and E.~N. Brown, ``Estimating a state-space model from point process
  observations,'' \emph{Neural Comp.}, vol.~15, no.~5, pp. 965--991, 2003.

\bibitem{paninski2004maximum}
L.~Paninski, ``Maximum likelihood estimation of cascade point-process neural
  encoding models,'' \emph{Network: Comp. in Neural Systems}, vol.~15, no.~4,
  pp. 243--262, 2004.

\bibitem{paninski2007statistical}
L.~Paninski, J.~Pillow, and J.~Lewi, ``Statistical models for neural encoding,
  decoding, and optimal stimulus design,'' \emph{Progress in brain research},
  vol. 165, pp. 493--507, 2007.

\bibitem{pillow2011model}
J.~W. Pillow, Y.~Ahmadian, and L.~Paninski, ``Model-based decoding, information
  estimation, and change-point detection techniques for multineuron spike
  trains,'' \emph{Neural Comp.}, vol.~23, no.~1, pp. 1--45, 2011.

\bibitem{truccolo2005point}
W.~Truccolo, U.~T. Eden, M.~R. Fellows, J.~P. Donoghue, and E.~N. Brown, ``A
  point process framework for relating neural spiking activity to spiking
  history, neural ensemble, and extrinsic covariate effects,'' \emph{Journal of
  neurophysiology}, vol.~93, no.~2, pp. 1074--1089, 2005.

\bibitem{haykin2008adaptive}
S.~S. Haykin, \emph{Adaptive filter theory}.\hskip 1em plus 0.5em minus
  0.4em\relax Pearson Education India, 2008.

\bibitem{eden2004dynamic}
U.~T. Eden, L.~M. Frank, R.~Barbieri, V.~Solo, and E.~N. Brown, ``Dynamic
  analysis of neural encoding by point process adaptive filtering,''
  \emph{Neural comp.}, vol.~16, no.~5, pp. 971--998, 2004.

\bibitem{donoho2006compressed}
D.~L. Donoho, ``Compressed sensing,'' \emph{Information Theory, IEEE Trans.
  on}, vol.~52, no.~4, pp. 1289--1306, 2006.

\bibitem{candes2006compressive}
E.~J. Cand{\`e}s \emph{et~al.}, ``Compressive sampling,'' in \emph{Proceedings
  of the International Congress of Mathematicians}, vol.~3.\hskip 1em plus
  0.5em minus 0.4em\relax Madrid, Spain, 2006, pp. 1433--1452.

\bibitem{candes2008introduction}
E.~J. Cand{\`e}s and M.~B. Wakin, ``An introduction to compressive sampling,''
  \emph{Signal Processing Magazine, IEEE}, vol.~25, no.~2, pp. 21--30, 2008.

\bibitem{Negahban}
S.~N. Negahban, P.~Ravikumar, M.~J. Wainwright, and B.~Yu, ``A unified
  framework for high-dimensional analysis of {M}-estimators with decomposable
  regularizers,'' \emph{Statistical Science}, vol.~27, no.~4, pp. 538--557,
  2012.

\bibitem{babadi2010sparls}
B.~Babadi, N.~Kalouptsidis, and V.~Tarokh, ``{SPARLS}: The sparse {RLS}
  algorithm,'' \emph{Signal Processing, IEEE Trans. on}, vol.~58, no.~8, pp.
  4013--4025, 2010.

\bibitem{kalouptsidis2011adaptive}
N.~Kalouptsidis, G.~Mileounis, B.~Babadi, and V.~Tarokh, ``Adaptive algorithms
  for sparse system identification,'' \emph{Signal Processing}, vol.~91, no.~8,
  pp. 1910--1919, 2011.

\bibitem{dumitrescu2012greedy}
B.~Dumitrescu, A.~Onose, P.~Helin, and I.~T{\u{a}}bu{\c{s}}, ``Greedy sparse
  {RLS},'' \emph{Signal Processing, IEEE Trans. on}, vol.~60, no.~5, pp.
  2194--2207, 2012.

\bibitem{fritz2003rapid}
J.~Fritz, S.~Shamma, M.~Elhilali, and D.~Klein, ``Rapid task-related plasticity
  of spectrotemporal receptive fields in primary auditory cortex,''
  \emph{Nature neuroscience}, vol.~6, no.~11, pp. 1216--1223, 2003.

\bibitem{Brown_pp}
W.~Truccolo, U.~T. Eden, M.~R. Fellows, J.~P. Donoghue, and E.~N. Brown, ``A
  point process framework for relating neural spiking activity to spiking
  history, neural ensemble, and extrinsic covariate effects,'' \emph{Journal of
  neurophysiology}, vol.~93, no.~2, pp. 1074--1089, 2005.

\bibitem{brown_func_conn}
Z.~Chen, D.~F. Putrino, S.~Ghosh, R.~Barbieri, and E.~N. Brown, ``Statistical
  inference for assessing functional connectivity of neuronal ensembles with
  sparse spiking data,'' \emph{Neural Systems and Rehabilitation Engineering,
  IEEE Trans. on}, vol.~19, no.~2, pp. 121--135, 2011.

\bibitem{needell2009cosamp}
D.~Needell and J.~A. Tropp, ``{CoSaMP}: Iterative signal recovery from
  incomplete and inaccurate samples,'' \emph{Applied and Computational Harmonic
  Analysis}, vol.~26, no.~3, pp. 301--321, 2009.

\bibitem{Toeplitz}
J.~Haupt, W.~U. Bajwa, G.~Raz, and R.~Nowak, ``Toeplitz compressed sensing
  matrices with applications to sparse channel estimation,'' \emph{Information
  Theory, IEEE Trans. on}, vol.~56, no.~11, pp. 5862--5875, 2010.

\bibitem{javanmard2014confidence}
A.~Javanmard and A.~Montanari, ``Confidence intervals and hypothesis testing
  for high-dimensional regression,'' \emph{The Journal of Machine Learning
  Research}, vol.~15, no.~1, pp. 2869--2909, 2014.

\bibitem{van2014asymptotically}
S.~Van~de Geer, P.~B{\"u}hlmann, Y.~Ritov, and R.~Dezeure, ``On asymptotically
  optimal confidence regions and tests for high-dimensional models,'' \emph{The
  Annals of Stat.}, vol.~42, no.~3, pp. 1166--1202, 2014.

\bibitem{zhang2014confidence}
C.-H. Zhang and S.~S. Zhang, ``Confidence intervals for low dimensional
  parameters in high dimensional linear models,'' \emph{Journal of the Royal
  Statistical Society: Series B (Statistical Methodology)}, vol.~76, no.~1, pp.
  217--242, 2014.

\bibitem{meinshausen2006high}
N.~Meinshausen and P.~B{\"u}hlmann, ``High-dimensional graphs and variable
  selection with the lasso,'' \emph{The Annals of Stat.}, pp. 1436--1462, 2006.

\bibitem{brown2002time}
E.~N. Brown, R.~Barbieri, V.~Ventura, R.~E. Kass, and L.~M. Frank, ``The
  time-rescaling theorem and its application to neural spike train data
  analysis,'' \emph{Neural comp.}, vol.~14, no.~2, pp. 325--346, 2002.

\bibitem{haslinger2010discrete}
R.~Haslinger, G.~Pipa, and E.~Brown, ``Discrete time rescaling theorem:
  determining goodness of fit for discrete time statistical models of neural
  spiking,'' \emph{Neural comp.}, vol.~22, no.~10, pp. 2477--2506, 2010.

\bibitem{fritz2005active}
J.~Fritz, M.~Elhilali, and S.~Shamma, ``Active listening: task-dependent
  plasticity of spectrotemporal receptive fields in primary auditory cortex,''
  \emph{Hearing research}, vol. 206, no.~1, pp. 159--176, 2005.

\bibitem{mesgarani2010computational}
N.~Mesgarani, J.~Fritz, and S.~Shamma, ``A computational model of rapid
  task-related plasticity of auditory cortical receptive fields,''
  \emph{Journal of computational neuroscience}, vol.~28, no.~1, pp. 19--27,
  2010.

\bibitem{gandolfo2000cortical}
F.~Gandolfo, C.-S. Li, B.~Benda, C.~P. Schioppa, and E.~Bizzi, ``Cortical
  correlates of learning in monkeys adapting to a new dynamical environment,''
  \emph{Proceedings of the National Academy of Sciences}, vol.~97, no.~5, pp.
  2259--2263, 2000.

\bibitem{van_de_geer}
S.~A. van~de Geer, ``On {H}oeffding's inequality for dependent random
  variables,'' in \emph{{Empirical Process Techniques for Dependent Data}},
  H.~Dehling and W.~Philipp, Eds.\hskip 1em plus 0.5em minus 0.4em\relax
  Springer, 2001.

\bibitem{bickel2009simultaneous}
P.~J. Bickel, Y.~Ritov, and A.~B. Tsybakov, ``Simultaneous analysis of {L}asso
  and {D}antzig selector,'' \emph{The Annals of Stat.}, pp. 1705--1732, 2009.

\bibitem{hoeffding1963probability}
W.~Hoeffding, ``Probability inequalities for sums of bounded random
  variables,'' \emph{Journal of the American statistical association}, vol.~58,
  no. 301, pp. 13--30, 1963.

\bibitem{scholkopf2001learning}
B.~Scholkopf and A.~J. Smola, \emph{Learning with kernels: {S}upport vector
  machines, regularization, optimization, and beyond}.\hskip 1em plus 0.5em
  minus 0.4em\relax MIT press, 2001.

\bibitem{figueiredo2005bound}
M.~A. Figueiredo and R.~D. Nowak, ``A bound optimization approach to
  wavelet-based image deconvolution,'' in \emph{Image Processing, 2005. ICIP
  2005. IEEE International Conf. on}, vol.~2.\hskip 1em plus 0.5em minus
  0.4em\relax IEEE, 2005, pp. II--782.

\bibitem{figueiredo2007majorization}
M.~A. Figueiredo, J.~M. Bioucas-Dias, and R.~D. Nowak,
  ``Majorization--minimization algorithms for wavelet-based image
  restoration,'' \emph{Image Processing, IEEE Trans. on}, vol.~16, no.~12, pp.
  2980--2991, 2007.

\bibitem{hunter2004tutorial}
D.~R. Hunter and K.~Lange, ``A tutorial on {MM} algorithms,'' \emph{The
  American Statistician}, vol.~58, no.~1, pp. 30--37, 2004.

\bibitem{figueiredo2003algorithm}
M.~A. Figueiredo and R.~D. Nowak, ``An {EM} algorithm for wavelet-based image
  restoration,'' \emph{Image Processing, IEEE Trans. on}, vol.~12, no.~8, pp.
  906--916, 2003.

\bibitem{bruckstein2009sparse}
A.~M. Bruckstein, D.~L. Donoho, and M.~Elad, ``From sparse solutions of systems
  of equations to sparse modeling of signals and images,'' \emph{SIAM review},
  vol.~51, no.~1, pp. 34--81, 2009.

\bibitem{daubechies2004iterative}
I.~Daubechies, M.~Defrise, and C.~De~Mol, ``An iterative thresholding algorithm
  for linear inverse problems with a sparsity constraint,''
  \emph{Communications on Pure and Applied Mathematics}, vol.~57, no.~11, pp.
  1413--1457, 2004.

\end{thebibliography}
}

\end{document}